\documentclass[twoside]{article}
\usepackage[accepted]{aistats2024}
\usepackage[english]{babel}

\usepackage{natbib}

\usepackage{amsfonts,amsmath,amsthm,booktabs}
\usepackage{bm}

\newcommand{\Exp}{\mathbb E}

\usepackage{graphicx}
\usepackage[colorlinks=true, allcolors=blue]{hyperref}
\usepackage{parskip}
\usepackage{xcolor}
\usepackage[inline,nomargin,index]{fixme}

\fxsetup{theme=color,mode=multiuser,inlineface=\itshape,envface=\itshape}
\FXRegisterAuthor{sv}{asv}{\colorbox{gray!10!white}{\color{black}Suresh}}
\FXRegisterAuthor{lk}{alk}{\colorbox{blue!10!white}{\color{black}Lizzie}}
\FXRegisterAuthor{cc}{acc}{\colorbox{red!10!white}{\color{black}Cyrus}}
\fxusetargetlayout{color}

\usepackage[utf8]{inputenc}

\usepackage{amsmath,amsfonts,amsthm}

\usepackage{amssymb}

\usepackage{mathtools}

\usepackage{thmtools}
\usepackage{thm-restate}

\usepackage{bm}
\usepackage{bbm}

\usepackage{xcolor}

\usepackage{tikz}
\usepackage{pgfplots}

\usepackage{pgfplotstable}

\usetikzlibrary{arrows,arrows.meta}
\usetikzlibrary{positioning,fit}
\usetikzlibrary{patterns}
\usetikzlibrary{decorations.pathmorphing,decorations.text,shapes.geometric}
\usetikzlibrary{fadings}
\usetikzlibrary{spy,backgrounds}
\usetikzlibrary{hobby}

\usepgfplotslibrary{fillbetween}

\usepackage{nicefrac}

\usepackage{multicol}

\usepackage{relsize}

\usepackage{commath}

\usepackage[inline]{enumitem}

\usepackage{upgreek}

\usepackage{yfonts}

\usepackage{pifont}

\usepackage[integrals]{wasysym}

\usepackage{microtype}

\usepackage{multirow}

\usepackage{subcaption}

\usepackage{url}

\usepackage{hyperref} 

\usepackage[noabbrev]{cleveref}

\usepackage[leftcaption]{sidecap}

\usepackage{algorithm} 

\usepackage{algpseudocode}

\usepackage{placeins}

\usepackage{appendix}

\DeclareMathOperator*{\argmin}{argmin}

\newcommand{\R}{\mathbb{R}}
\newcommand{\RNN}{\R_{\mathsmaller{0+}}}

\DeclareMathOperator*{\Expect}{\mathbb{E}}

\DeclareMathOperator*{\Prob}{\mathbb{P}}

\newcommand{\x}{\vec{x}}

\newcommand{\lv}{\mathcal{S}}
\newcommand{\epsv}{\bm{\varepsilon}}
\newcommand{\etav}{\bm{\eta}}
\newcommand{\wv}{\bm{w}}

\newcommand{\distributed}{\thicksim}

\newcommand{\X}{\mathcal{X}}
\newcommand{\Y}{\mathcal{Y}}

\newcommand{\HC}{\mathcal{H}}

\newcommand{\Risk}{\mathrm{R}}
\newcommand{\ERisk}{\hat{\Risk}}

\newcommand{\Malfare}{\vphantom{\mathrm{W}}\raisebox{1.53ex}{\rotatebox{180}{\ensuremath{\mathrm{W}}}}}

\newcommand{\ProbDist}{\mathcal{D}}

\newcommand{\vsigma}{\bm{\sigma}}

\newcommand{\frange}{r}

\newcommand{\loss}{\ell}

\DeclareFontShape{OMX}{cmex}{m}{b}{<-> cmexb10}{}
\SetSymbolFont{largesymbols}{bold}{OMX}{cmex}{m}{b}

\providecommand{\LandauO}{\bm{\mathrm{O}}} 

\providecommand{\LandauTheta}{\bm{\Uptheta}}
\providecommand{\LandauOmega}{\bm{\Upomega}}

\newcommand{\cyrus}[1]{\textcolor{green!50!black}{Cyrus: #1}}

\iftrue

\newtheorem{definition}{Definition}

\newtheorem{lemma}[definition]{Lemma}
\newtheorem{theorem}[definition]{Theorem}

\fi

\newcommand{\wrt}{w.r.t.}

\newcommand{\iid}{i.i.d.}
\newcommand{\whp}{w.h.p.}

\newcommand{\NGroups}{g}

\DeclareMathOperator{\VC}{VC}

\makeatletter
\providecommand{\leftsquigarrow}{%
  \mathrel{\mathpalette\reflect@squig\relax}%
}
\newcommand{\reflect@squig}[2]{%
  \reflectbox{$\m@th#1\rightsquigarrow$}%
}
\makeatother

\newcommand{\Rade}{\textgoth{R}}
\newcommand{\ERade}{\hat{\Rade}}

\tikzset{
  declare function={
      sgn(\x) = (and(\x<0, 1) * -1) + (and(\x>0, 1) * 1);
      gm2(\va,\vb) = (sqrt(\va * \vb ));
      pm2(\va,\vb,\p) = ((\va ^ \p + \vb ^ \p) / 2) ^ (1 / \p);
      pmw2(\va,\vb,\wa,\wb,\p) = (\wa * \va ^ \p + \wb * \vb ^ \p) ^ (1 / \p);
      amean3(\va,\vb,\vc) = ((\va + \vb + \vc) / 3);
      as3(\va,\vb,\vc,\pp) = ((\va ^ \pp + \vb ^ \pp + \vc ^ \pp) / 3);
      pm3(\va,\vb,\vc,\pp) = as3(\va,\vb,\vc,\pp) ^ (1 / \pp);
      asw3(\va,\vb,\vc,\wa,\wb,\wc,\pp) = (\wa * \va ^ \pp + \wb * \vb ^ \pp + \wc * \vc ^ \pp);
      pmw3(\va,\vb,\vc,\wa,\wb,\wc,\pp) = asw3(\va,\vb,\vc,\wa,\wb,\wc,\pp) ^ (1 / \pp);
    },
}

\usepackage{multicol,multirow}

\newif\iflongversion
\longversionfalse

\expandafter\def\expandafter\normalsize\expandafter{%
    \normalsize%
    \setlength\abovedisplayskip{4pt}%
    \setlength\belowdisplayskip{2pt}%
    \setlength\abovedisplayshortskip{-8pt}%
    \setlength\belowdisplayshortskip{2pt}%
}

\begin{document}

\twocolumn[

\aistatstitle{To Pool or Not To Pool: Analyzing the Regularizing Effects of Group-Fair Training on Shared Models}

\aistatsauthor{ Cyrus Cousins \And I. Elizabeth Kumar \And  Suresh Venkatasubramanian }

\aistatsaddress{ University of Massachusetts Amherst \And  Brown University \And Brown University } ]
\begin{abstract}
In fair machine learning, one source of
performance disparities between groups
is overfitting to groups with relatively few training samples.
We derive group-specific bounds on the generalization
error
of
welfare-centric fair machine learning 
that benefit from
the larger sample size of the majority group. We do this by considering
group-specific Rademacher averages over a restricted hypothesis class,
which contains the family of models likely to perform well with respect to a
fair learning objective (e.g., a power-mean). Our simulations demonstrate
these bounds improve over a na\"ive method, as expected by theory, with particularly significant improvement for smaller group sizes.
\end{abstract}


\section{INTRODUCTION}
It is well-known that learned models can have performance or outcome disparities on underrepresented or disadvantaged groups in a distribution \citep{buolamwini2018gender, obermeyer2019dissecting}.
Research suggests that these disparities are the result of a complex interaction between the training procedure, model class, and training data \citep{chen_why}.

Group-based welfare-centric machine learning attempts to mitigate disparities by 
optimizing 
\emph{aggregations of per-group risk values}, rather than average overall loss.
In other words, the task is to approximate $\argmin_{h \in \HC} \Malfare\bigl( \Risk(h, \ProbDist_1), \dots, \Risk(h, \ProbDist_{\NGroups}) \bigr)$ for some
\emph{malfare function} $\Malfare(\cdot)$, where
$\Risk(h, \ProbDist_{i})$ is the risk (average loss) of
group 
$i$
under
model
$h$.
Such objectives 
produce models that fairly compromise
among groups in various ways.
The
malfare function determines the fairness concept; for example, $\wv$-weighted risk minimization is equivalent to optimizing \emph{utilitarian malfare} $\Malfare_{1}(\lv; \wv) = \lv \cdot \wv$, and
taking $\Malfare(\cdot)$ to be the \emph{maximum
} produces the \emph{minimax-optimal} $h^{*}$
,
a.k.a., the \emph{egalitarian} or \emph{Rawlsian} fair model.

However, 
training with \emph{empirical risk} is susceptible to ``overfitting to fairness,’’ wherein models
overfit
small or high-risk \emph{minority groups}. 
\Citet{cousins2021axiomatic,cousins2022uncertainty,cousins2023revisiting}  shows
that
generalization error (
overfitting) of the \emph{overall objective} decreases
with \emph{each group's} sample size,
but
the current SOTA generalization 
bounds 
\emph{for group $i$}
depend only on \emph{group $i$}'s sample size. 
We address this discrepancy;
in particular, we show that
in fair learning,
each group $i$
effectively learns
over a ``restricted 
class’’ 
of models that are reasonably likely 
given the training data \emph{for
all
groups} $j \neq i$, thus we bound their generalization error 
via Rademacher averages \emph{of the restricted class}, improving over existing bounds based on the original hypothesis class. 


We begin by introducing notation and preliminary concepts (\cref{sec:bg:prelim}) and situating our approach with respect to existing literature (\cref{sec:bg:related}). We derive group-specific bounds on the generalization 
error
of jointly trained models, which
benefit from the larger sample size of the majority group (\cref{sec:bounds}).
These 
techniques also translate to improved bounds on the generalization error of the malfare objective itself.
Additionally, we experimentally verify our methods
on synthetic linear and logistic regression tasks, finding 
that our bounds better describe the overfitting behavior of fair-learning methods than 
SOTA analysis (\cref{sec:exp}).
Our analysis allows us to resolve key real-world problems, such as when multiple groups benefit from pooling data to train a single (shared) model. 
All proofs are relegated to \cref{appx:sec:proofs}.


\section{BACKGROUND}
\label{sec:bg}

We now introduce notation and preliminary concepts, followed by a brief review of related work.

\subsection{Preliminaries}
\label{sec:bg:prelim}

We assume a standard supervised learning setting.
Given domain label space $\Y$, domain $\X$, and codomain $\Y'$, we assume a \emph{hypothesis class} $\HC \subseteq \X \to \Y'$ and \emph{loss function} $\loss: \Y' \times \Y \to \R$.
Now, suppose a sample $(\bm{x}, \bm{y}) = \bm{z} \in (\X \times \Y)^{m}$ or instance distribution $\ProbDist$ over $\X \times \Y$.
We
define the \emph{empirical risk} of 
hypothesis $h$ as
\[
\ERisk(h, \bm{z}) \doteq \frac{1}{m} \sum_{i=1}^{m} \loss(h(\bm{x}_{i}), \bm{y}_{i})
\enspace,
\]
and the true risk over the distribution $\ProbDist$ as
\[
\Risk(h, \ProbDist) \doteq \Expect_{(x,y) \distributed \ProbDist} [ \loss(h(x), y) ]
\enspace.
\]

A standard supervised learning task then 
identifies the \emph{empirical risk minimizer}
\[
\hat{h} \doteq \argmin_{h \in \HC} \ERisk(h, \bm{z})
\]
as a proxy for the \emph{true risk minimizer}
\[
h^{*} \doteq \argmin_{h \in \HC} \Risk(h, \ProbDist)
\enspace.
\]

This framework encapsulates simple supervised settings where $\Y = \Y'$, such as 
least-squares regression or hard binary classification, but it also
contains
more sophisticated supervised learning settings, like probabilistic classification or conditional density estimation. 

\paragraph{Group-Fair Learning}
This work considers \emph{group-fair learning}, in which we assume not one instance distribution $\ProbDist$, but rather $\NGroups$ \emph{per-group} instance distributions $\ProbDist_{1:\NGroups}$, and per-group samples $\bm{z}_{i} \distributed \ProbDist_{i}^{\smash{\bm{m}_{i}}}$ where $\bm{m}_{i}$ is the sample size for group $i$.
The distribution $\ProbDist_{i}$ encapsulates the situations encountered by members of each group $i$, which may vary in 
$\X$ (situations encountered by each group), as well as their conditional labels $\Y | \X$ (responses or labels to a given situation).

To treat groups fairly, we consider objectives that consider the risk of all groups.
In particular, we assume a cardinal \emph{malfare function} $\Malfare(\cdot): \R^{\NGroups} \to \R$, and
we then seek the \emph{empirical malfare minimizer}
\begin{align*}
\hat{h} &\doteq \argmin_{h \in \HC} \Malfare \bigl(i \mapsto \ERisk(h, 
\bm{z}_{i}
)
\bigr) \\
 &= \argmin_{h \in \HC} \Malfare \bigl( \ERisk(h, \bm{z}_{1}), \ERisk(h, \bm{z}_{2}), \dots, \ERisk(h, \bm{z}_{\NGroups}) \bigr)
\end{align*}
as a proxy for the \emph{true malfare minimizer}
\[
h^{*} \doteq \argmin_{h \in \HC} \Malfare \bigl(i \mapsto \Risk(h, \ProbDist_{i}) \bigr)
\enspace.
\]

\paragraph{On Malfare Functions}

The choice of malfare function $\Malfare(\cdot)$ directly encodes how one wishes to make tradeoffs between various groups at various levels of risk.
The malfare function is thus a fundamental
fair-learning
hyperparameter that must be selected to achieve a modeler’s desired fairness properties, i.e., choosing a malfare function is equivalent to choosing a fairness concept.

Two popular choices are the \emph{utilitarian malfare} (weighted average), which generally weights the risk of each group proportional to their size, and the \emph{egalitarian malfare}, which seeks to lift up the most disadvantaged groups by minimizing the maximum risk.
These are in some sense two extremes of a spectrum: utilitarian malfare
only weights
groups, and does not distinguish between high-risk and low-risk groups (no equitable redistribution), whereas egalitarian malfare considers only the risk of each group, offering preferential treatment to those most in need (no consideration of non-minimal groups).
It is known that both of the above malfare functions belong to a general class of such functions.
\begin{definition}[Power-Mean Malfare]
Suppose some 
$p \geq 1$, positive \emph{probability measure} $\wv \in \triangle_{\NGroups}$, and nonnegative \emph{
risk vector} $\lv \in \RNN^{\NGroups}$.
We define the \emph{weighted power-mean} as
\begin{equation}
\label{eq:wpmean}
\!\!\!
\Malfare_{p}(\lv; \wv) \doteq 
\vphantom{\sum}
{\sqrt[p]{
  \vphantom{\sum^{,}_{\cdot}}\smash{\sum_{i=1}^{\NGroups}}
  \wv_{i} \lv_{i}^{p} }}
  \enspace, \;\,
\Malfare_{\infty}(\lv; \wv) \doteq 
  \, \smash{\max_{\mathclap{i \in 1, \dots, \NGroups}}} \, \lv_{i}
\enspace. \!
\end{equation}

\end{definition}


Both
utilitarian and egalitarian malfare arise as 
power-mean special-cases 
$p=1$ and 
$p=\infty$, respectively.

The power-mean class is axiomatically justified \citep{debreu1959topological,gorman1968structure,cousins2021axiomatic,cousins2023revisiting}, which motivates its use in a variety of learning and allocation settings \citep{barman2020tight,cousins2022faire3,viswanathan2023general,cousins2023dividing,cousins2023good
}.
Fairness and robustness are closely linked, and \citet{cousins2023algorithms} also motivates power-means, as well as Gini malfare, and other malfare classes, from the perspective of robustness.
This work is neutral to the choice of malfare function; 
we only seek to show that our methods may be applied to any malfare concept that meets certain broad criteria.

We generally assume 
that $\Malfare(\cdot)$ is \emph{monotonic}, i.e., that increasing any group's risk
never decreases
malfare.
Furthermore, 
\emph{convex} malfare functions are convenient for optimization, and in \cref{sec:bounds:linear} we utilize this property to efficiently bound Rademacher averages. 
Finally, several of our bounds
have
algebraically convenient corollaries if we assume 
\emph{Lipschitz continuity}, i.e., small changes to risk
yield
small changes to malfare.
%
The power-mean malfare family, as well as other 
malfare classes, such as
the Gini class
\citep{weymark1981generalized,gajdos2005multidimensional} or
the \emph{utilitarian-maximin}
class
\citep{deschamps1978leximin,bossert2020axiomatization,schneider2020utilitarian},
each arise uniquely from 
their own
sets of
axioms.
Each assume 
some type of \emph{monotonicity},
\emph{transfer principles}, such as the Pigou-Dalton \citep{pigou1912wealth,dalton1920measurement}, which incentivize \emph{equitable redistribution} of harm 
and give rise to convexity, as well as 
some concept of 
\emph{continuity}, which coupled with functional analysis of the resultant class, give rise to Lipschitz continuity.
Our criteria for malfare functions are thus quite reasonable.

\paragraph{Statistical Background}

The Rademacher average is a key statistical tool used to bound the \emph{supremum deviation} of empirical means from their expectations \citep{bartlett2002rademacher}.
Denote the \emph{loss class} $\loss \circ \HC \doteq \left\{ (x, y) \mapsto \loss(h(x), y) \, \middle| \, h \in \HC \right\}$, and define Rademacher averages as follows.
\begin{definition}[Rademacher Averages]
\label{def:rade}
Let 
$\vsigma_{1:m}$ be a
vector
of $m$
\iid\ $\mathrm{Unif}(\pm 1)$ random variables. 
The \emph{empirical Rademacher average} is then 
\[
\ERade_{m} (\ell \circ \mathcal{H}, \bm{z} ) \doteq \Expect_{\vsigma} \left[ \sup_{h \in \HC} \frac{1}{m} \sum_{i=1}^m \vsigma_i \ell(h(\bm{x}_{i}), \bm{y}_{i}) \right]
\enspace,
\]
i.e., the maximum correlation of any $h \in \HC$ with noise on a sample $\bm{z} \in (\X \times \Y)^{m}$,
and the \emph{Rademacher average} is 
its expectation over
\iid\ 
samples
from $\ProbDist$, i.e.,
\[
\Rade_{m} (\ell \circ \mathcal{H}, \ProbDist ) \doteq \Expect_{\bm{z} \distributed \ProbDist^{m}} \bigl[ \, \ERade_{m}(\loss \circ \HC, \bm{z}) \, \bigr]
\enspace.
\]
\end{definition}

Assuming \emph{bounded loss range} $\frange$, 
let
$\epsv_{i} \doteq \frange\sqrt{\frac{\ln \frac{1}{\delta}}{2\bm{m}_{i}}}$ and
$\hat{\etav}_{i} \doteq 2\ERade_{\bm{m}_{i}}(\loss \circ \HC, \bm{z}_{i}) + 2\epsv_{i}$.
For any failure probability $\delta$
,
$h \in \HC$, and group $i$, 
sampling error is bounded as
\begin{equation}
\label{eq:hoeff-bound}
\Prob_{\!\bm{z}_{i} \distributed \ProbDist_{i}^{\smash{\bm{m}_{i}}}\!} \left( \, \abs{ 
    \ERisk(h, \bm{z}_{i}) - \Risk(h, \ProbDist_{i}) 
    }
    > \epsv_{i} \right) < 2\delta 
\enspace.
\end{equation}
Moreover, considering all $h \in \HC$ simultaneously, 
we have for each group $i$ that
{%
\small%
\begin{equation}
\label{eq:rade-textbook}
\Prob_{\!\!\bm{z}_{i} \distributed \ProbDist_{i}^{\smash{\bm{m}_{i}}}\!} \! \left( \! \begin{matrix} 2\Rade_{m}(\loss \! \circ \! \HC, \ProbDist_{i}) > 2\ERade_{m}(\loss \! \circ \! \HC,\bm{z}_{i}) \! + \! \epsv_{i}  \\ \bigvee \sup\limits_{h \in \HC} \abs{ 
    \ERisk(h, \bm{z}_{i}) - \Risk(h, \ProbDist_{i})
    }
    > \hat{\etav}_{i}
    \end{matrix} \! \right) < 3\delta 
\!\enspace.
\end{equation}}%
Equations~\eqref{eq:hoeff-bound} and \eqref{eq:rade-textbook} are used throughout for various
hypotheses and hypothesis classes,
both in the above forms, and as 1-tailed variants.
These ``textbook results'' are now standard in learning theory\footnote{Constants vary between sources, depending on 
definitions and derivations.
Our probabilistic statements 
use 2-tailed Hoeffding \citeyearpar{hoeffding1963probability} 
bounds 
and 3 applications of McDiarmid's \citeyearpar{mcdiarmid1989method} inequality,
with
optimal
bounded differences, for 
Rademacher averages with no absolute value
inside the supremum. 
} \citep{shalev2014understanding,mitzenmacher2017probability}.

The quantity $\abs{\smash{\ERisk(h, \bm{z}_{i})} - \Risk(h, \ProbDist_{i})}$ of \eqref{eq:hoeff-bound} is the \emph{absolute deviation} between the empirical risk
and the expected risk for
each individual
$h \in \HC$,
and
it bounds
the \emph{estimation error} (i.e., error due to sampling) of any such function. 
The quantity $\sup_{h \in \HC} \abs{ 
\smash{\ERisk(h, \bm{z}_{i})} - \Risk(h, \ProbDist_{i}) }$ of \eqref{eq:rade-textbook} is known as the \emph{supremum deviation} over
the loss class $\loss \circ \HC$, and
it bounds
the \emph{generalization error}, both due to sampling error
and due to
selection bias
(training), of the learned $\hat{h}$.

From 
\eqref{eq:rade-textbook} and a union-bound over groups, following \citet{cousins2021axiomatic,cousins2022uncertainty,cousins2023revisiting},
we probabilistically bound each group's generalization error (training-true
risk gap) as
\begin{equation}
\label{eq:group-gen}
\Prob_{\bm{z}_{1:\NGroups}} \! \left( \forall i\!:\, 
    \abs{ 
    \ERisk(\hat{h}, \bm{z}_{i}) - \Risk(\hat{h}, \ProbDist_{i})
    }
    \leq \hat{\etav}_{i} \right) \geq 1 - 3\NGroups\delta
\enspace.
\end{equation}
Moreover, if $\Malfare(\lv)$ is monotonically increasing in $\lv$, 
then the malfare generalization
error obeys
{\small
\begin{equation}
\label{eq:m-gen}
\!\!\Prob_{\mathclap{\bm{z}_{1:\NGroups}}} \left( \begin{matrix}{\Malfare\bigl( i \mapsto \ERisk(\hat{h}, \bm{z}_{i}) \! - \! \hat{\etav}_{i} \bigr)} & \!\! {\text{(Empcl.\ LB)}} \\ \leq  {\Malfare\bigl( i \mapsto \Risk(\hat{h}, \ProbDist_{i}) \bigr)} & \!\! {\mathclap{\text{(True Malfare)}}} \ \  \\ \leq \! {\Malfare\bigl( i \mapsto \ERisk(\hat{h}, \bm{z}_{i}) \! + \! \hat{\etav}_{i} \bigr)} \!\!\!\!\! & \!\! {\text{(Empcl.\ UB)}} \end{matrix} \, \right) \! \geq \! 1 - 3\NGroups\delta
\enspace,
\end{equation}%
}%
i.e., the true malfare of
$\hat{h}$ is sandwiched by upper and lower bounds in terms of empirical malfare.
Finally, using also a union bound over \eqref{eq:hoeff-bound},
the
gap
between the true risk of the empirical malfare minimizer $\hat{h}$ and the true malfare minimizer $h^{*}$
is
\begin{equation}
\label{eq:m-subopt}
\hspace{-0.4cm}\Prob_{\bm{z}_{1:\NGroups}} \! \left( \begin{matrix} \Malfare\bigl( i \mapsto \Risk(\hat{h}, \ProbDist_{i}) - \hat{\etav}_{i} \bigr) \\ \leq \Malfare\bigl( i \mapsto \Risk(h^{*}\!, \ProbDist_{i}) + \epsv_{i} \bigr) \end{matrix} \right) \geq 1 - 5\NGroups\delta
\enspace.
\hspace{-0.25cm}
\end{equation}

\subsection{Related work}
\label{sec:bg:related}


This work
follows others in group-based welfare-centric fair machine learning.
This 
often takes the form of \emph{Rawlsian} or \emph{egalitarian} learning, also known as \emph{minimax fair learning}, wherein $\Malfare(\cdot)$ is the maximum function, and the goal is to minimize the maximum (over groups) average loss \citep{diana2021minimax,shekhar2021adaptive,abernethy2022active,martinez2020minimax,lahoti2020fairness,cortes2020agnostic,shekhar2021adaptive,
dong2022decentering}%
\iftrue%
, which is 
a form of \emph{distributionally robust optimization} \citep{hu2018does,oren2019distributionally,sagawa2019distributionally}%
\fi%
. Most such works only consider performance over the training set, but the Seldonian learner framework \citep{thomas2019preventing} explicitly
requires trained models be probably approximately optimal \wrt\ some constrained nonlinear objective. 
Similarly, the fair-PAC learning framework \citep{cousins2021axiomatic,cousins2023revisiting}
considers
malfare minimization with power-mean objectives.

Due to the nonlinearity of $\Malfare(\cdot)$, existing work bounds generalization errror \emph{separately} for each group $j$, and applies assumed Lipschitz or H\"older continuity properties of $\Malfare(\cdot)$ to bound the overall objective \citep{cousins2021axiomatic,cousins2022uncertainty,cousins2023revisiting}.
In this work, we show sharper bounds on the generalization error of
malfare
objectives,
but we also seek
to bound
each group's
generalization error.

\paragraph{Multitask learning}
There is
overlap between group fair learning (GFL) and multitask learning (MTL).
This work
shows
that GFL 
reduces generalization error for all groups (particularly smaller groups), which is essentially the motivation for MTL. 
In both cases, we have
$\NGroups$ distributions 
(per-group in GFL, per-task in MTL) and some objective that considers each distribution through $\Risk(h, \ProbDist_{i})$.
To our knowledge, there is no published work in multitask learning
on objectives that treat tasks nonlinearly, i.e., the objective is always \citep{caruana1997multitask,zhang2018overview,zhang2021survey} 
\[
\hat{h} \doteq \argmin_{h \in \HC} \sum_{i=1}^{\NGroups} \frac{1}{\bm{m}_i} \sum_{j=1}^{\bm{m}_i} \loss_{i}(h(\bm{x}_{i,j}), \bm{y}_{i,j}) \enspace. 
\]

Existing 
MTL analysis bounds 
generalization error 
by considering all data at once \citep{zhang2020generalization,zhang2021survey};
assuming $m$ samples each for $\NGroups$ groups,
VC dimension, Rademacher averages, 
etc.\ 
bound total estimation error as
$O  \sqrt{{\ln \mathsmaller{\frac{1}{\delta}}}/{m\NGroups}}$.
Such methods do not apply in our setting, as we seek \emph{per-group generalization bounds} and treat \emph{nonlinear objectives}, thus ultimately we do not expect bounds of this order.

\paragraph{To pool or not to pool}
Some work directly addresses the tradeoff between training pooled versus separate models for groups.
\citet{dwork2018decoupled} define the \emph{cost-of-coupling} as 
\[
\max_{\ProbDist} \left( \min_{h \in \HC}\sum_{i=1}^{\NGroups}\Risk(h, \ProbDist_{i}) -
\sum_{i=1}^{\NGroups} \min_{h \in \HC} \Risk(h, \ProbDist_{i})
    \right) \enspace,
\]
i.e., worst-case difference between the sum risk of the optimal shared model $\hat{h}$, vs.\ sum risk of optimal per-group models $\hat{h}_{1:\NGroups}$.
When this quantity is positive,
training with pooled data may require tradeoffs in accuracy across groups.
They then introduce \emph{transfer learning} methods to train per-group classifiers $\hat{h}_{1:\NGroups}$ while leveraging available data where appropriate.
Similarly to our work, this results in improved VC-theoretic groupwise bounds on generalization error than fully separated training.
However, the goal of our learning framework is still to learn a joint model, avoiding thorny questions of disparate treatment.
\citet{benefit_of_splitting} also examine the tradeoff, where the metric of interest or \emph{benefit of splitting} is based on an egalitarian notion of fairness.
They
largely focus on 
the infinite-samples or known-distributions settings;
however, 
they provide VC-theoretic generalization bounds on the benefit of splitting.
These are necessarily worst-case (over possible distributions), and specific to binary classification, whereas we provide data-dependent Rademacher average bounds applicable to a broad range of supervised and unsupervised settings.


\section{BOUNDING GENERALIZATION ERROR 
IN FAIR TRAINING
}
\label{sec:bounds}

The generalization error analysis of \cref{sec:bg:prelim} does not take into account the fact
that learning is not equally likely to produce any $h \in \HC$. In this section, we present a sharper analysis that reflects this, both in per-group generalization error bounds, and in the overall generalization error of a malfare objective.

Our approach is to take the core idea of localization \citep{bartlett2005local} --- restricting the function class of interest to a subset that with high probability contains the function that will be learned --- and generalize it to apply in multi-group fair learning settings.
In \Cref{sec:bounds:theoretical} we argue that, for each group $i$, with high probability, the learned function $\hat{h}$ belongs to some $\HC_{i}^{*} \subseteq \HC$. 
We 
bound generalization error over $\HC_{i}^{*}$, 
with $\ERade_{\bm{m}_{i}}(\HC_{i}^{*}, \bm{z}_{i})$, 
where often $\ERade_{\bm{m}_{i}}(\HC_{i}^{*}, \bm{z}_{i}) \ll \ERade_{\bm{m}_{i}}(\HC, \bm{z}_{i})$. The analysis depends on the group index $i$, since while analyzing group $i$, we can treat the training samples $\bm{z}_{j}$ as constant for each $j \neq i$, but the class $\HC_{i}^{*}$ must not depend on $\bm{z}_{i}$ for vital technical reasons (see proof of \cref{thm:bounds-empirical}; 
we require $\HC_{i}^{*}$ to be established independently from $\bm{z}_{i}$ in $\ERade_{\bm{m}_{i}}(\HC_{i}^{*}, \bm{z}_{i})$). We thus establish a theoretical hypothesis class that directly depends on the training sample for each $j \neq i$, but depends on the distribution for group $i$ instead of its training sample.


Unfortunately, $\HC_{i}^{*}$ is a theoretical object (not actually known, as it depends on $\ProbDist_{i}$). Thus, we have little recourse but to relax to dependence on purely empirical quantities. We thus establish in \Cref{sec:bounds:empirical} an empirical class $\hat{\HC}_{i}$, which depends on $\bm{z}_{i}$ instead of $\ProbDist_{i}$. At first glance this seems to violate core statistical precepts, but through careful construction, we show that $\hat{\HC}_{i}$ acts merely as a probabilistic proxy for $\HC_{i}^{*}$.

Finally, we must actually \emph{estimate} the relevant Rademacher bounds. In \Cref{sec:bounds:linear} we illustrate how this can be done for linear hypothesis classes using Monte-Carlo Rademacher averaging.

\subsection{Theoretical Restricted 
Classes}
\label{sec:bounds:theoretical}


When
bounding
the generalization error of group $i$, we want to construct a restricted hypothesis class leveraging information given by the remaining group samples, in particular their empirical risks.
However, 
we can't directly use the group $i$ sample $\bm{z}_{i}$, so instead we bound empirical risk $\ERisk(h, \bm{z}_{i})$ 
in terms of $\Risk(h, \ProbDist_{i})$.
Intuitively, we want this restricted class to be the set of all $h \in \HC$ that could reasonably be the function 
we learn from \emph{all data} (the empirical malfare minimizer $\hat{h}$), where the restricted class is constructed after observing only the data $\bm{z}_{j}$ for all groups $j \neq i$.

Similar techniques are common in learning theory and the study of localization, where a \emph{theoretical class} 
is constructed based on the (unknown) distribution(s), and subsequently an \emph{empirical class} that is with high probability a superset which can be built from the data.
Our approach, however, is unique in that it is in some sense \emph{half-empirical}, as the theoretical class depends on the distribution $\ProbDist_{i}$ of one group, and the samples $\bm{z}_{j}$ from all groups $j \neq i$.
We do this instead of 
constructing a ``fully theoretical'' class using only the distributions $\ProbDist_{1:\NGroups}$, as well as an empirical variant based on all training samples $\bm{z}_{1:\NGroups}$, which would be substantially larger.

First, let $\epsv_{i} \doteq \frange\sqrt{\frac{\ln \frac{1}{\delta}}{2\bm{m}_{i}}}$ 
and
$\etav_{i} \doteq 2\Rade_{\bm{m}_{i}}(\loss \circ \HC, \ProbDist_{i} ) + \epsv_{i}$,
where $\bm{m}_{i}$ is the sample size for group $i$ and $\frange$ is the range of loss values in $\loss \circ \HC$.
%
Recall that the empirical malfare minimization task 
is to select
\[
\hat{h} \doteq \argmin_{h' \in \HC} \Malfare\!\left( j \mapsto \smash{\ERisk}(h', \bm{z}_{j}) 
\right) \enspace,
\]
but since we can't observe sample $i$ yet, we (pessimistically) upper-bound the objective value (w.h.p.) as
{
\begin{align*}
&
\inf_{{h' \in \HC}} \Malfare\!\left( j \mapsto \smash{\ERisk}(h', \bm{z}_{j}) ; \wv \right) \leq \\
 &
 \inf_{{h' \in \HC}} \Malfare\!\left( j \mapsto \begin{cases} j \! \neq \! i \!\! & \smash{\ERisk}(h                       ', \bm{z}_{j}) \\ j \! = \! i \!\! & \Risk(h', \ProbDist_{i}) + \epsv_{i} \end{cases} 
	 \right) \enspace,
\end{align*}}%
and (optimistically) lower-bound the empirical malfare of all $h \in \HC$, w.h.p.\ simultaneously, as
{
\begin{align*}
&\Malfare\!\left( j \mapsto \smash{\ERisk}(h, \bm{z}_{j}) ; \wv \right) \geq \\
 &\Malfare\!\left( j \mapsto \begin{cases} j \! \neq \! i \!\! & \smash{\ERisk}(h, \bm{z}_{j}) \\ j \! = \! i \!\! & \Risk(h, \ProbDist_{i}) - \etav_{i} \end{cases} 
    \right) \enspace.
\end{align*}}%

Via this analysis, we then construct our theoretical class, which with high probability shall contain the empirical malfare minimizer $\hat{h}$,
as
the subset $\HC_{i}^{*} \subseteq \HC$ constrained to $h$ such that
{
\begin{align}
    &\Malfare\!\left( j \mapsto \begin{cases} j \! \neq \! i \!\! & \ERisk(h, \bm{z}_{j}) \\ j \! = \! i \!\! & \Risk(h, \ProbDist_{i}) - \etav_{i}^{} \end{cases} 
	 \right) \leq \notag \\ &\inf_{\!h' \in \HC\!} \Malfare\!\left( j \mapsto \begin{cases} j \! \neq \! i \!\! & \ERisk(h', \bm{z}_{j}) \\ j \! = \! i \!\! & \Risk(h', \ProbDist_{i}) + \epsv_{i} \end{cases} \label{eq:theoretical-restricted-class}
	 \right)
  \enspace.
\end{align}%
}%
This construction is valid (formalized in \cref{thm:bounds-theoretical}), as we took any $h$ that optimistically could outperform a pessimistic estimate of the empirical objective.

The
LHS is a ``best case'' estimate of the empirical malfare of a candidate hypothesis, whereas the
RHS
is a ``worst case'' estimate of the minimal empirical malfare, because we want our restricted hypothesis class to be large enough to
contain
any $h \in \HC$ that might be the empirical malfare minimizer.
In particular, the LHS uses a Rademacher average bound
\eqref{eq:rade-textbook},
as the
bound must apply to all $h \in \HC$, but a simple tail-bound term
\eqref{eq:hoeff-bound}
suffices on the RHS, as we are comparing to a bound involving some specific $h'$ (not dependent on the data $\bm{z}_{i}$).

Intuitively,
for utilitarian malfare,
$\hat{\HC}_{i}$ describes models that definitely perform well for
groups $j \neq i$,
and will probably perform well for group $i$. 
Some malfare functions, such as power-means, are undefined for negative risk values, and the LHS
risk lower bounds $\Risk(h, \ProbDist_{i}) - \etav_{i}^{}$
may be negative. 
However, if we assume risk (or loss) is nonnegative,
we may use the risk lower bound $\max(0, \Risk(h, \ProbDist_{i}) - \etav_{i}^{})$,
which preserves convexity, continuity, and even differentiability if $p < \infty$ 
except around the point $\bm{0}$.

Observe now that, conditioning on $\bm{z}_{j}$ for each $j \neq i$, with high probability over choice of $\bm{z}_{i}$, empirical malfare minimization yields some $\hat{h} \in \HC_{i}^{*}$.
Therefore, for all intents and purposes, learning occurs over $\HC_{i}^{*}$, and we may thus use Rademacher averages over this restricted class to bound generalization error for group $i$.
%
Formally put, we have the following result.

\begin{restatable}[Theoretical Group-Regularized Malfare Bounds]{theorem}{thmboundstheoretical}
\label{thm:bounds-theoretical}
Suppose a monotonic 
\emph{malfare function} \linebreak[3]$\Malfare(\cdot): \R^{\NGroups} \to \R$, \emph{hypothesis class} $\HC \subseteq \X \to \Y'$, \emph{loss function} $\loss: \Y' \times \Y \to \R$, per-group \emph{distributions} $\ProbDist_{1:\NGroups}$ over $\X \times \Y$, and per-group \emph{samples} $\bm{z}_{1:\NGroups}$, with $\bm{z}_{j} \distributed \ProbDist_{j}^{\bm{m}_{j}}$ for each group $j$. 
Fix any group index $i$, and take $\HC_{i}^{*}$ defined as in \eqref{eq:theoretical-restricted-class}.
The following then hold. 
\begin{enumerate}[wide,labelwidth=0pt,labelindent=0pt,nosep]
\item With probability at least $1 - 2\delta$ over choice of $\bm{z}_{i}$, it holds that
$\hat{h} \in \HC_{i}^{*}$.
\item With probability at least $1 - 4\delta$ over choice of $\bm{z}_{i}$, 
\[
\abs{ \Risk(\hat{h}, \ProbDist_{i}) - \ERisk(\hat{h}, \bm{z}_{i}) } \leq 2\Rade_{\bm{m}_{i}}(\loss \circ \HC_{i}^{*}, \ProbDist_{i} ) + \epsv_{i} \enspace.
\]
\end{enumerate}
\end{restatable}



\subsection{
Empirical Restricted 
Classes}
\label{sec:bounds:empirical}
$\HC_{i}^{*}$ is an object only of theoretical interest (it is not actually known, since it depends on $\ProbDist_{i}$).
Consequently, without more information, $\Rade_{\bm{m}_{i}}(\loss \circ \HC_{i}^{*}, \ProbDist_{i})$, and thus the bounds of \cref{thm:bounds-theoretical}, can not be computed.
\iftrue
\begin{figure}
\centering
\vspace{4pt}
\begin{tikzpicture}[
    scale=0.95,
    hclass/.style={draw=black,very thick},
    ehclass/.style={hclass,dashed,fill=red,draw opacity=0.6,fill opacity=0.2,decorate,decoration={random steps,segment length=6pt,amplitude=1pt}},
    ]

\draw[hclass,fill=yellow,fill opacity=0.3] (0, -0.25) circle (3.35);
\node (hs0) at (0, 2.65) { \begin{tabular}{c} \large$\HC$ 
\end{tabular}};

\path[ehclass,draw,use Hobby shortcut,closed=true, fill=red,fill opacity=0.3,xscale=1.2] (2, 0) .. (0, -2.2) .. (-1.6, -2) .. (-1.8, 0) .. (0, 1.1);

\path[ehclass,draw,use Hobby shortcut,closed=true, fill=red,fill opacity=0.3] (2, 0) .. (0, -2.2) .. (-1.6, -2) .. (-2, 0) .. (0, 2);

\path[ehclass,draw,use Hobby shortcut,closed=true,xscale=1.3,yscale=1.2] (1.2, 0) .. (0, -1.5) .. (-0.8, -1) .. (-1.1, 0) .. (0, 1.6) .. (0.2, 1.7);


\node (ehs1) at (0, -3) { \begin{tabular}{c}  Samples of $\hat{\HC}_{i}$ 
\end{tabular}};

\draw (ehs1) -- (0.5, -2.1);
\draw (ehs1) -- (-0.5, -2.3);
\draw (ehs1) -- (0, -1.8);

\draw[hclass,fill=blue!30!white] (0, 0) circle (1);
\node (hs1) at (0, 0.18) { \begin{tabular}{c} $\HC_{i}^{*}$ 
\end{tabular}};

\node (hh) at (0.36, -0.58) {\small \raisebox{0.2ex}{$\bm{\cdot}$} \!\! $\hat{h}$};

\node (hs) at 
    (0.8, 1.4) {\small ${h}^{*}$ \!\!\!\! \raisebox{-0.2ex}{$\bm{\cdot}$}}; 

\end{tikzpicture}
\vspace{2pt}
\caption{
{Visualization of unrestricted class $\HC$, 
theoretical restricted class $\HC_{i}^{*} $, 
and samples of empirical restricted class $\hat{\HC}_{i}$ (varying $\bm{z}_{i}$).
One possible empirical malfare minimizer $\hat{h}$ (contained by $\hat{\HC}_{i}$ and $\HC_{i}^{*}$ with high probability), as well as the true malfare minimzer $h^{*}$ (which may fall outside of $\HC_{i}^{*}$ or $\hat{\HC}_{i}$ due to overfitting to groups other than $i$) are also shown.%
}}
\label{fig:theory-vis}
\end{figure}
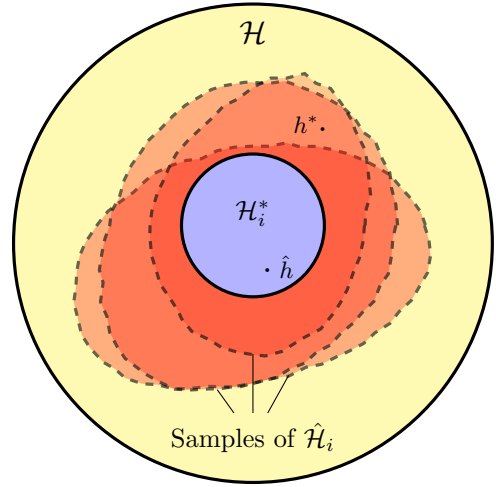
\fi
We remedy this issue here, relaxing dependence on the distribution $\ProbDist_{i}$ by replacing it with dependence on the training sample $\bm{z}_{i}$ and thus establishing a new \emph{empirically restricted hypothesis class}.

Note that we can't simply substitute $\ERisk(h, \bm{z}_{i})$ for $\Risk(h, \ProbDist_{i})$, as \cref{thm:bounds-theoretical} 
clearly requires the restricted hypothesis class $\HC_{i}^{*}$ to be fixed before observing the training data $\bm{z}_{i}$.
We account for this by indirectly using $\ERisk(h, \bm{z}_{i})$ to bound $\Risk(h, \ProbDist_{i})$. 
%
In particular, 
take $\hat{\etav}_{i} \doteq 2\ERade_{\bm{m}_{i}}(\loss \circ \HC, \bm{z}_{i} ) + 2\epsv_{i}$,
and take $\epsv_{i} \doteq \frange\sqrt{\frac{\ln \frac{1}{\delta}}{2\bm{m}_{i}}}$, as in \eqref{eq:rade-textbook}. 
Now, we construct our empirical class $\hat{\HC}_{i}$, which with high probability shall contain the theoretical class $\HC_{i}^{*}$,
as
the subset $\hat{\HC}_{i} \subseteq \HC$ constrained to $h$ such that
{
\begin{align}
\label{eq:empirical-restricted-class}
&
\Malfare\!\left( j \mapsto \begin{cases} j \! \neq \! i \!\! & \ERisk(h, \bm{z}_{j}) \\ j \! = \! i \!\! & \ERisk(h, \bm{z}_{i}) - 2\hat{\etav}_{i} \end{cases} 
    \right) \leq \notag \\ &
    \inf_{\!h' \in \HC\!} \Malfare\!\left( j \mapsto \begin{cases} j \! \neq \! i \!\! & \ERisk(h', \bm{z}_{j}) \\ j \! = \! i \!\! & \ERisk(h', \bm{z}_{i}) + 2\epsv_{i} \end{cases} 
    \right) \enspace.
\end{align}
%
Note that \eqref{eq:empirical-restricted-class} matches \eqref{eq:theoretical-restricted-class}, except risks and Rademacher averages are 
bounded in terms of their empirical counterparts.
In particular, on the LHS, \whp, for all $h \in \HC$ it holds
$\ERisk(h, \bm{z}_{i}) - 2\hat{\etav_{i}} \leq \Risk(h, \ProbDist_{i}) - \etav_{i}$,
and on the RHS, \whp,
$\ERisk(h', \bm{z}_{i}) + 2\epsv_{i} \geq \Risk(h', \ProbDist_{i}) + \epsv_{i}$.
%
\iftrue
\Cref{fig:theory-vis} visualizes the difference between $\hat{\HC}_{i}$ and ${\HC}_{i}^{*}$, as well as other key players.
\fi

Observe now that, with high probability, $\HC_{i}^{*} \subseteq \hat{\HC}_{i}$, therefore we can employ the theoretical properties of $\HC_{i}^{*}$ while being able to compute everything from a sample using $\hat{\HC}_{i}$.
The following theorem makes precise this statement, and should be viewed as an \emph{empirical counterpart} to \cref{thm:bounds-theoretical}.

\begin{restatable}[Empirical Group-Regularized Malfare Bounds]{theorem}{thmboundsempirical}
\label{thm:bounds-empirical}
Suppose as in \cref{thm:bounds-theoretical}.
The following then hold for $\hat{\HC}_{i}$ defined as in \eqref{eq:empirical-restricted-class}.

\begin{enumerate}[wide,labelwidth=0pt,labelindent=0pt,nosep]
\item With probability at least $1 - 4\delta$ over choice of $\bm{z}_{i}$, it holds that $\hat{h} \in \HC_{i}^{*} \subseteq \hat{\HC}_{i}$.

\item With probability at least $1 - 6\delta$, it holds that
\[
\abs{ \Risk(\hat{h}, \ProbDist_{i}) - \ERisk(\hat{h}, \bm{z}_{i}) } \leq 2\ERade_{\bm{m}_{i}}(\loss \circ \hat{\HC}_{i}, \bm{z}_{i} ) + 2\epsv_{i} \enspace.
\]
\end{enumerate}

\end{restatable}

\Cref{thm:bounds-empirical} satisfies our primary goal of showing per-group generalization bounds for fair learning that leverage information from other groups.
In particular,
when $\hat{\HC}_{i} \subset \HC$,
we obtain sharper generalization bounds, which 
quantifies the intuition that 
training a shared model is less susceptible to overfitting than training 
per-group models.
\Cref{thm:bounds-empirical}~part~2 should be contrasted with \eqref{eq:group-gen}, which gives a similar guarantee 
using Rademacher averages of the \emph{unrestricted class} $\HC$.
\Cref{coro:emp-malfare-bounds} now applies these bounds to improve the state-of-the-art generalization guarantees for (nonlinear) malfare objectives,
which would otherwise depend on Rademacher averages of $\hat{\HC}_{i}$ rather than $\HC$,
cf.\
\eqref{eq:m-gen}.

\begin{restatable}[Empirical Malfare Generalization Bounds]{corollary}{coroempmalfarebounds}
\label{coro:emp-malfare-bounds}
Suppose as in \cref{thm:bounds-empirical}.
Suppose also that there exists some $\lambda > 0$ and norm $\norm{\cdot}_{\Malfare}$ such that $\Malfare(\cdot)$ is $\lambda$-$\norm{\cdot}_{\Malfare}$ Lipschitz continuous,
i.e., $\forall \lv,\lv'$:
$
\Malfare(\lv + \lv') \leq \Malfare(\lv) + \lambda \norm{\lv'}_{\Malfare}
$.
We then have:
\begin{enumerate}[wide,labelwidth=0pt,labelindent=0pt,nosep]
\item 
With probability at least $1 - 5\NGroups\delta$, the true malfare of $\hat{h}$ is bounded by 
\\
\begin{minipage}{0.48\textwidth}
{\small
\begin{align*}
\Malfare\!\left(j \mapsto \Risk(\smash{\hat{h}}, \ProbDist_{j}) 
    \right)
\hspace{-2.08cm} & \\
&\leq \Malfare\!\left(j \mapsto \ERisk(\smash{\hat{h}}, \bm{z}_{j}) + 2\ERade_{\bm{m}_{j}}(\hat{\HC}_{j}, \bm{z}_{j}) + 2\epsv_{j} 
    \right) \\
&\leq \Malfare\!\left(j \mapsto \ERisk(\smash{\hat{h}}, \bm{z}_{j}) \! \right) + \lambda\norm{j \mapsto 2\ERade_{\bm{m}_{j\!}}(\hat{\HC}_{j}, \bm{z}_{j}) + 2\epsv_{j}}_{\Malfare\!}
\enspace.
\end{align*}
}
\end{minipage}
\item
With probability at least $1 - 6\NGroups\delta$, we bound the suboptimality of $\smash{\hat{h}}$ as \\
\begin{minipage}{0.48\textwidth}{
\small
\begin{align*}
\Malfare\left(j \mapsto \Risk(\smash{\hat{h}}, \ProbDist_{j}) 
    \right) \hspace{-0.85cm} & \\
 &\leq \Malfare\left(j \mapsto \Risk(h^{*}\!, \ProbDist_{j}) + 2\ERade_{\bm{m}_{j}}(\hat{\HC}_{j}, \bm{z}_{j}) + 3\epsv_{j} 
    \right) \\
\ \implies \ \; \abs{ \vphantom{\ERisk} \Malfare\left(j \mapsto \Risk(h^{*}\!, \ProbDist_{j}) 
    \right) - \Malfare\left(j \mapsto \Risk(\smash{\hat{h}}, \ProbDist_{j}) 
    \right) } \hspace{-5.25cm} & \\
 &\leq \lambda\norm{j \mapsto 2\ERade_{\bm{m}_{j}}(\hat{\HC}_{j}, \bm{z}_{j}) + 3\epsv_{j}}_{\Malfare}
\enspace.
\end{align*}
}
\end{minipage}
%
\end{enumerate}
\end{restatable}


The first inequality of parts 1 \& 2 of \cref{coro:emp-malfare-bounds} is 
sharper, but the second is generally more analytically convenient.
\iflongversion
\else
In particular, any power-mean malfare function 
$\Malfare_{p}(\cdot; \wv)$
obeys
\begin{equation}
\label{eq:lip-subadd}
\renewcommand{\!}{\hspace{-0.025em}}
\hspace{-0.2cm}
\Malfare_{p}(\lv \! + \! \lv'\!; \wv) \! - \! \Malfare_{p}(\lv; \wv) \leq \Malfare_{p}(\lv'\!; \wv) \leq \norm{\lv'}_{\infty} \!\!\!\hspace{-0.25em} \enspace, \!\!\!
\end{equation}
thus we bound malfare generalization error in terms of
the generalization error of each group.
\fi

Naturally, one may ask how sharp this localization strategy is.
We now show 
an example where \cref{thm:bounds-empirical} improves slow $\LandauO(\frac{1}{\sqrt{m}})$ convergence rates to fast $\LandauO(\frac{1}{m})$ convergence rates.
%
Consider unit-range 0-dimensional linear regression, 
i.e.,
mean estimation under square loss $\loss$, with $\NGroups = 1$.
Thus we have
\[
\loss \circ \HC_{r} = \left\{ \loss( h_{c}(x), y) = (c - y)^{2} \, \middle| \, c \in [-r, r] \right\}
\]
with $r = 1$. 
Take \emph{constant probability distribution} $\ProbDist = 0$, thus $\bm{y} = \bm{0}$. 
From random walk theory, we have
\begin{align*}
\ERade_{m}(\loss \circ \HC_{r}, \bm{y}) &=
\Expect_{\vsigma} \left[ \sup_{c \in [-r, r]} \frac{1}{m}\sum_{i=1}^{m} \vsigma_{i} (\bm{y}_{i} - c)^{2} \right] \\
 &= \Expect_{\vsigma} \left[ \sup_{c \in [-r, r]} \frac{1}{m}\sum_{i=1}^{m} \vsigma_{i} c^{2}  \right] \\
 &= \frac{r^{2}}{2} {\Expect_{\vsigma} \left[ \, \abs{ \frac{1}{m}\sum_{i=1}^{m}  \vsigma_{i} } \right]} 
 \approx r^{2} 
 {\sqrt{\frac{1}{2\pi m}}}
 \enspace.
\end{align*}
%
To 
construct $\hat{\HC}$, observe that we have $\ERisk(c, \bm{y}) = c^{2}$, thus via \eqref{eq:empirical-restricted-class} we restrict s.t.\ $c^{2} \leq 4\ERade_{m}(\loss \circ \HC_{r}, \bm{y}) + 6\varepsilon \approx \sqrt{\frac{8}{\pi m}} + 6\sqrt{\frac{\ln \frac{1}{\delta}}{ 2m }} \implies 
\abs{c} \leq r \in \LandauTheta \sqrt[4]{\frac{1}{m}}$.
We thus have
\[
\ERade_{m}(\loss \circ \hat{\HC}, \bm{y}) \approx \frange^{2} \mathsmaller{\sqrt{\frac{1}{2\pi m}}} \in \LandauTheta\left( \mathsmaller{\frac{1}{m}} \right)
\enspace,
\]
which asymptotically improves
$\ERade_{m}(\loss \circ {\HC}, \bm{y}) \approx \sqrt{\frac{1}{2\pi m}}$.

\subsection{Monte-Carlo Rademacher Averages of Linear Hypothesis Classes}
\label{sec:bounds:linear}

We now present a method to estimate Rademacher averages for linear hypothesis classes using Monte-Carlo sampling. We start by noting that, in general if $\ell(\hat{y}, y) = f(g(\hat{y}, y))$ and $f$ is $\lambda$-Lipschitz-continuous, then we have, for any $\bm{z} \in (\X \times \Y)^{m}$, that
\begin{equation}
\label{eq:contraction}
    \ERade_{m}(\ell \circ \HC, \bm{z} ) \leq \lambda \ERade_{m}(g \circ \HC, \bm{z} )
    \enspace.
\end{equation}

For this reason, we formulate the Rademacher averages of both linear least-squares regression and logistic regression as follows. Take  $\HC = \left\{ h_{\bm{\beta}}(\bm{x}) \doteq \bm{\beta} \cdot \bm{x} \, \middle| \, \bm{\beta} \in \bm{B} \right\}$ and loss function $\ell(\hat{y}, y) = f(g(\hat{y}, y))$, where for least-squares regression, $g(\hat{y},y) = \hat{y} - {y}$ and $f(u) = u^2$. This is $\lambda$-Lipschitz continuous, assuming bounded $\bm{B}$, $\mathcal{X}$, and $\mathcal{Y}$, with $\lambda = 2 \sup_{\bm{B}, \mathcal{Y}, \mathcal{X}} |\bm{x} \cdot \bm{\beta} - y|$.\footnote{In practice, we compute the Lipschitz constant over $\HC$, rather than over $\hat{\HC}_i \subseteq \HC$, which would require computing the diameter of $\hat{\HC}_i$ or bounding the range of $g \circ \hat{\HC}_{i}$.
}
For logistic regression, in which $\mathcal{Y} = \pm 1$, we have $g(\hat{y},y) = \hat{y} \cdot y$ and $f(u) = \ln(1 + \exp(u) )$, which is
1-Lipschitz. 

\paragraph{Estimation}
Standard methods for bounding Rademacher averages of linear regression classes start by bounding the Rademacher average of $\HC$ itself \citep{shalev2014understanding}. However, this method is loose \citep{cousins2020sharp}, and seems especially so for irregular weight spaces (i.e., those not defined by simple $p$-norms), which known analytic methods can not handle. 

Instead, we directly estimate the Rademacher average of the function family $g \circ \HC$ directly using Monte-Carlo estimation. That is to say, given sampled Rademacher random variables $\vsigma \in (\pm 1)^{n \times m}$ and data sample $\bm{z} \in (\X \times \Y)^{m}$, we compute
\begin{minipage}{0.48\textwidth}
{\small\renewcommand{\!}{}
\begin{equation}
\label{eq:lin-mcera}
\!\!\ERade{}^n_{m}(g \circ \HC,  \bm{z}; \vsigma) \! \doteq \! \frac{1}{n} \smash{ \sum_{k=1}^n  \sup_{\bm{\beta} \in W} \frac{1}{m} \! \sum_{j=1}^{m} } \vsigma_{k,j} g( \bm{x}_{j} \cdot \bm{\beta}, \bm{y}_j ) \enspace.
\end{equation}
}
\end{minipage}
This fully data-dependent method gracefully tolerates arbitrary data distributions and 
parameter spaces, and is loose only in a small amount of Monte-Carlo error and the contraction inequality \citep{cousins2020sharp}.
In practice, we use $\lambda \ERade{}^n_{\bm{m}_i}(g \circ \HC,  \bm{z}_i; \vsigma)$ as a plug-in estimate of $\lambda \ERade{}_{\bm{m}_i}(g \circ \HC,  \bm{z}_i)$, which then bounds $\ERade_{\bm{m}_{i}}(\ell \circ \HC, \bm{z}_{i} )$ via \eqref{eq:contraction}. 
We similarly estimate and bound Rademacher averages over our restricted hypothesis classes 
as $\lambda \ERade{}^n_{\bm{m}_i}(g \circ \hat{\HC}_i,  \bm{z}_i; \vsigma)$.

\begin{restatable}[Convex Optimization for Monte-Carlo Rademacher Averages]{lemma}{lemmacvx}
\label{lemma:cvx}

Suppose the \emph{parameter space} $\bm{B}$ of $\HC$ is a convex set, \emph{loss} $\loss(h_{\bm{\beta}}(\bm{x}), y)$ is convex in $\bm{\beta} \in \bm{B}$ for all $\bm{x} \in \X$, $y \in \Y$, and \emph{malfare} $\Malfare(\cdot): \R^{\NGroups} \to \R$ is quasiconvex and
monotonically increasing in each argument.
Then the parameter spaces of $\hat{\HC}_{i}$ and $\HC^{*}_{i}$ are 
convex sets.

Moreover, if $g \circ \HC$ is an affine function family, then $\ERade{}^n_{m}(g \circ \hat{\HC}_{i}, \bm{z}_{i}; \vsigma)$ 
reduces to maximizing a linear function over a convex set. 
%
Similarly, if we strengthen the quasiconvexity assumption on $\Malfare(\cdot)$ to \emph{convexity}, then 
EMM reduces to minimizing a convex objective over the convex set $\bm{B}$.
%
%
\end{restatable}

\paragraph{Visualizing $\hat{\HC}_i$ in least-squares regression}

\begin{figure}
    \centering

\includegraphics[width=\columnwidth]{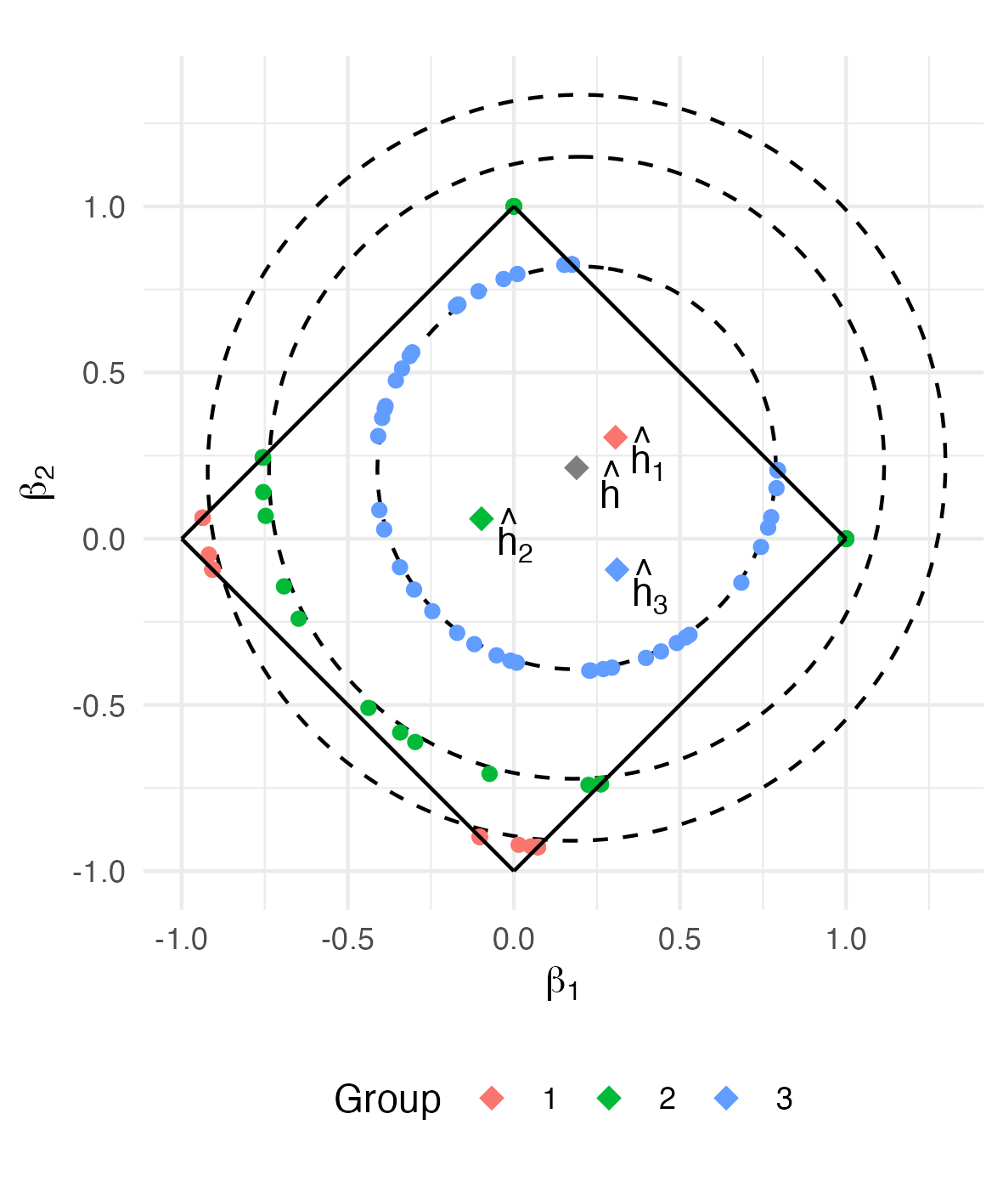}

\vspace{-15pt}
    \caption{Rademacher average samples in the parameter space of $\hat{\HC}_i$ for each group $i \in \{1, 2, 3\}$.
    }
%
\label{fig:contours}
\end{figure}

\begin{table}[t]
\caption{Sample sizes $\bm{m}_{1:3}$, parameter vectors $\bm{\beta}_{1:3}$, and Monte-Carlo empirical Rademacher averages (MCERA) for both $\HC$ and $\smash{\hat{\HC}_{1:3}}$.}
\vspace{-8pt}
\label{table:contours}

\centering

\begin{tabular}{cr|c|ll}
\small Group & $\bm{m}_{i}$ & True $\bm{\beta}$ & \multicolumn{2}{c}{ MCERA } \\ 
ID & & & \ \ $\HC$ & \ \ $\hat{\HC}_{i}$ \\
\hline
1 & 6500 & (0.3,0.3) & 0.047 & 0.046 \\
2 & 3000 & (-0.1,0.1) & 0.075 & 0.046 \\
3 & 500 & (0.3,0) & 0.183 & 0.135
\end{tabular}
\end{table}

For least-squares regression, under utilitarian malfare, the restricted hypothesis constraint of $\hat{\HC}_{i}$ is an ellipsoid (under egalitarian welfare, it is an \emph{intersection} of ellipsoids). We visualize a simple example in \cref{fig:contours}, with parameters and results described in \cref{table:contours}.

\ccnote{Math in 1 col submission}

Taking
$\bm{B} \doteq \left\{ \bm{\beta} \in \R^{2} \, \middle| \, \norm{\bm{\beta}}_{1} \leq 1 \right\}$
to be the unit $\ell_{1}$ ball, we 
sample $(\bm{x}, y)$ as \linebreak[3]$\bm{x} \sim \mathrm{Unif}([-1,1]^2)$, \linebreak[3]$y = \bm{x} \cdot \bm{\beta}_i + \mathrm{Unif}([-1,1])$, 
where each group has slightly different data generating parameters $\bm{\beta}_i$. 
In \cref{fig:contours}, taking $\delta=0.1$, we plot the $n=100$ values of $\bm{\beta}$ which
realize each
supremum
of \eqref{eq:lin-mcera} for some Rademacher sample $\vsigma_{k}$.
These points necessarily lie on either (the corner of) the $\ell_1$ constraint boundary of $\bm{B}$ or the restricted hypothesis constraint boundary of $\hat{\HC}_i$, illustrated by the concentric 
ellipses,
which represent constant upper-bounds of weighted utilitarian malfare over the whole dataset and are centered around $\hat{h}$.

Note that for the smallest group, 3, the fact that $\hat{h}$ must perform well on the other two groups under weighted utilitarian malfare shrinks $\hat{\HC}_3$ significantly. However, the generalization bound over the largest group, 1, is not significantly improved when taken over $\hat{\HC}_1$.
\ccnote{Punchline: for small sample sizes $\hat{\HC}_{i} = \HC^{*}_{i} = \HC$ (w.h.p.), after some point, $\HC^{*}_{i} \subset \HC$, and $\hat{\HC}_{i}$ follows.
Moreover, if $h^{*}$ is not on the boundary of $W$, and the class is uniformly convergent, for a sufficiently large sample size, the $W$ constraint does not matter. 
}
\lknote{future idea: Fix data for one or two groups, and see dynamics of sample size}

\section{EXPERIMENTS
}
\label{sec:exp}

We illustrate the utility of our results with some experiments. Our approach is to construct an example dataset where we can demonstrate a clear benefit (to minority groups) to pooled training, and then show how our refined generalization bounds are in fact sharper than standard Rademacher bounds. 
We do this by assuming that the individual distributions of the groups are similar enough that, for underrepresented minority groups, pooled training reduces generalization error.

Our experiments are based on a binary logistic regression task with $3$ groups.
Suppose the unit $\ell_{\infty}$ ball domain, i.e., $\X = [-1, 1]^{15}$, binary label space 
$\Y = \pm 1$, and parameter space 
\linebreak[3]$\bm{B} \doteq \left\{\bm{\beta} \in \mathbb{R}^{15} \, \middle| \,  \norm{ \bm{\beta} }_1 \leq 15 \right\}$.
For each group $i$, we generate samples $(\bm{x}, y)$ with $\bm{x} \sim \mathrm{Unif}(\mathcal{X})$,  \linebreak[3]$\Prob(y = 1) = \text{logistic}(\bm{x} \cdot \bm{\beta}_i + \xi)$, with noise \linebreak[2]$\xi \sim \mathcal{N}(0,0.1)$, for \linebreak[2]$\text{logistic}(u) = \frac{1}{1 + \exp(-u)}$.

We assume groupwise data generating parameters and a constant proportional composition of the full training sample as in \cref{tab:logistic-regression-data}. Notably, the data generating model for groups 1 and 3 are very similar, but there is always much more data available for group 1.

\begin{table}[b]
\caption{Data generating parameters for logistic regression experiments.}
\vspace{-8pt}
\begin{tabular}{l|ll}
        & Data proportion & True parameters \\
\hline
Group 1 & 75\%            & $\beta_i = 0.3$                \\
Group 2 & 20\%            & $\beta_i = 0.1$                  \\
Group 3 & 5\%             & $\beta_i = 0.2$     
\end{tabular}
\label{tab:logistic-regression-data}
\end{table}

\begin{figure}
\includegraphics[width=0.48\textwidth]{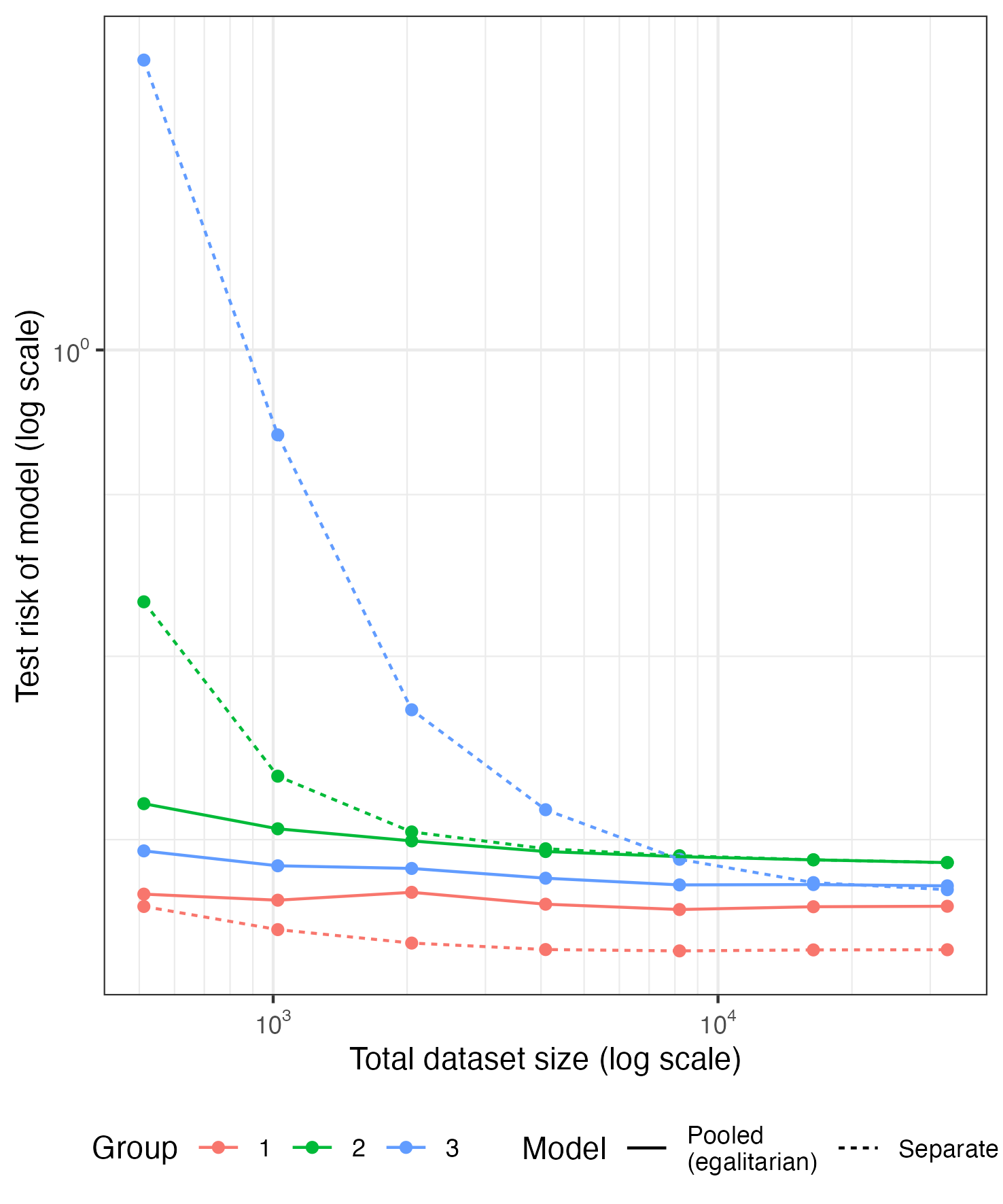}

\lknote{Made it longer. Shorter version also in figures folder (with the appendix "\_shorter") if it comes to that}
\vspace{-8pt}
    \caption{Average test risk of pooled and separately trained models on three groups 
    (see \cref{tab:logistic-regression-data}).
    }
    \label{fig:pooled-model-risk}
\end{figure}

\begin{figure*}
    \centering
\includegraphics[width=\textwidth]{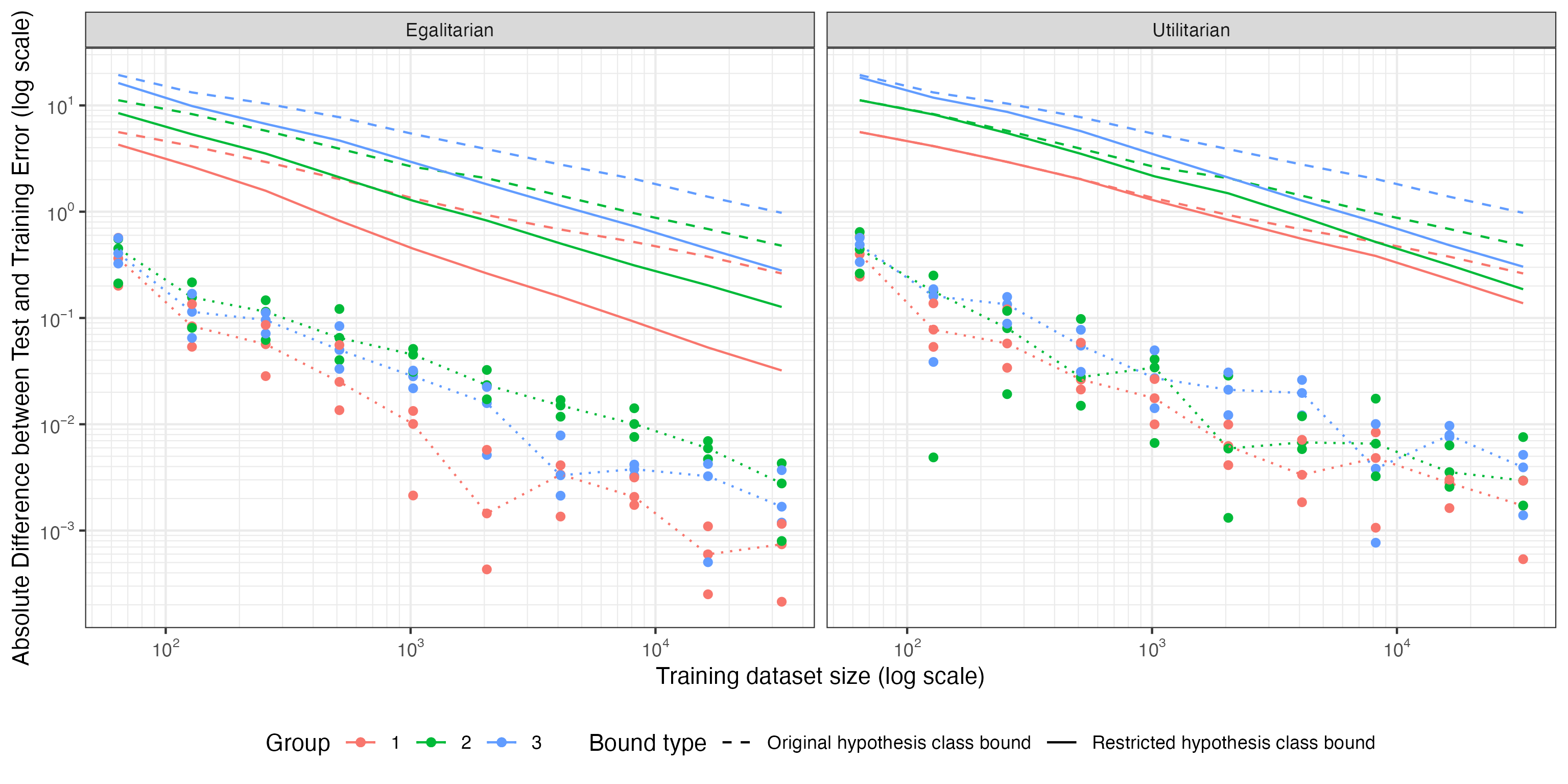}

    \caption{Generalization error bounds derived from original hypothesis class $\HC$ and restricted hypothesis classes $\hat{\HC}_i$, compared with 
    shared model $\hat{h}$
    train-test gap
    over 7 independent runs, with quartiles and 
    median trend lines.
    }
    \label{fig:bounds-improvement}
\end{figure*}

In \cref{fig:pooled-model-risk}, we plot the
average test risk  of each group $i$ over 7 independent runs for
malfare-minimizing models $\hat{h}$
or for risk-minimizing models $\hat{h}_{i}$ 
as a function of total training sample size, where
test risk is computed from a held-out test set 
with 20,000 samples for each group.
We observe that pooled models almost always have lower per-group test risks than the separately-trained models
$\hat{h}_{i}$
on the minority groups (2 and 3),
which we attribute to
the regularizing effect of pooled training overcoming the small discrepancies between the data generating parameters of each group (see~\cref{tab:logistic-regression-data}).
While the above describes small-sample behavior, for
sufficient sample sizes,
per-group models
should
dominate shared models, and we do observe this for group 3 with the maximum sample size of $32768 \cdot 0.05 \approx 1638$.

We then compute the bounds derived from Monte-Carlo Rademacher averages (with $\delta=0.1$) over samples $\bm{z}_i$ over both $\HC$ and $\hat{\HC}_i$ for each $i$ (\cref{fig:bounds-improvement}). 
Since the bounds derived from Rademacher averages over $\HC$ essentially function as bounds on the generalization error of the separately trained models, the fact that the bound over $\hat{\HC}_i$ is tighter correctly suggests that pooled training is better for the minority groups in this scenario, especially when using egalitarian training.

In the utilitarian case, we see that initially $\hat{\HC}_i$ bounds match $\HC$ bounds, but for sufficiently large sample sizes, they diverge.
In the egalitarian case, $\hat{\HC}_i$ bounds are always better than $\HC$ bounds, and they appear to decay at an asymptotically greater rate (slope on the log-log plot), reaching an order of magnitude improvement in the case of the largest group (group 1).
This suggests that our bounds characterize generalization error substantially more sharply than the na\"ive method.

\ccerror{TODO: bounds diverge, order of magnitude, why egal is better?}

\ccnote{
It is known that the generalization error usually behaves as
\[
\LandauOmega\!\left(\!\sqrt{\frac{v}{m}}\right)
    \leq \Expect_{
        \bm{z}_{1:\NGroups}}\left[ \abs{\Risk(\hat{h}, \ProbDist_{i}) - \ERisk(\hat{h}, \bm{z}_{i})} \right]
    \leq \LandauO\!\left(\!\sqrt{\frac{C}{m}}\right) \enspace,
\]
where $v > 0$ is the \emph{maximum loss variance} over any \emph{risk-optimal} $h^{*}$, and $C < \infty$ is a constant that depends on $\HC$ and $\ProbDist$.
In particular, letting $V$ denote the \emph{maximum loss variance} over \emph{all of $\HC$}, $C \leq \lim_{m \to \infty} V + \bigl(\sqrt{m}\Rade_{m}(\loss \circ \HC, \ProbDist) \bigr)^{2}$, and obeys $C \in \LandauO ( V \cdot \ln^{2} \abs{\HC} )$ or $C \in \LandauO ( V \cdot \VC^{2}(\HC) )$, i.e., the log-cardinality $\ln \abs{\HC}$ or VC dimension $\VC(\HC)$ (or similar quantities) measure the capacity of $\HC$ to overfit. 
\ccerror{Citations for fast learning rate theory}
\\
For sufficiently small sample sizes, we expect the restriction of our theory is vacuous, i.e., $\hat{\HC}_{i} = \HC$, and thus we see the \emph{slow learning rate} 
of $\frac{1}{\sqrt{m}}$ initially.
As the sample size $\bm{m}_{i}$ grows and constraints become non-vacuous, the simultaneous contraction of the hypothesis space $\hat{\HC}_{i}$ and growth of $\bm{m}_{i}$ can lead to a temporary \emph{fast convergence rate} of $\approx \frac{1}{m}$, though the above theory indicates that, so long as $v > 0$, this fast learning rate does not hold asymptotically; eventually the learning rate settles back to $\frac{1}{\sqrt{m}}$, albeit with error bounds a constant factor $\LandauTheta( \sqrt{\frac{v}{C}} )$ better than the na\"ive method.
\\
This hypothesis is borne out in \cref{fig:bounds-improvement}: under utilitarian welfare, bounds using $\hat{\HC}_{i}$ initially match those using $\HC$, until reaching roughly $\bm{m}_{1} = 1024$, $\bm{m}_{2} = 512$, and $\bm{m}_{3} = 256$ samples, after which the $\hat{\HC}_{i}$ bounds diverge from those using $\HC$.
Moreover, our log-log plot reveals a slope of roughly $-\frac{1}{2}$, i.e., a $\frac{1}{\sqrt{m}}$ power-law rate, for both the \emph{initial portion} of $\hat{\HC}_{i}$ bounds and $\HC$ bounds in their entirety. 
\lkerror{This feels like tea leaf reading to me}
In the egalitarian constraint experiments, we actually see that even at only $64$ samples, $\hat{\HC}_{i} \subset \HC$, and our bounds already
decay
more rapidly than the $\HC$ bounds.
}


\section{CONCLUSION}
\label{sec:conc}

We show that fair learning, like multitask learning, has a \emph{regularizing effect},
reducing overfitting to each group as compared to per-group models trained solely on 
their data. 
Concretely, we show that, from the perspective of each group, 
fair-learning (empirical malfare minimization) effectively occurs over some 
\emph{restricted hypothesis class}, and we the bound generalization error of each group's risk in terms of their Rademacher averages over these restricted classes.
This 
technique yields refined generalization bounds, not just for the overall learning, task, but also for the risk \emph{of each individual group}.

Such bounds are of particular importance in
learning settings where
minority groups often suffer poor model performance \citep{mehrabi2021survey}, such as medical ML \citep{obermeyer2019dissecting} and facial recognition \citep{buolamwini2018gender,cavazos2020accuracy}. %
Moreover,
in critical systems, having provable guarantees on the generalization error \emph{of each task}, rather than just the overall generalization error,
can greatly improve reliability and user trust.
This is also valuable in multi-task learning settings, where task fairness and task-specific bounds are of interest, e.g., in 
distributionally-robust LLMs \citep{oren2019distributionally}.

While the contributions of this paper are theoretical,
our setting is practically motivated.
Understanding the generalization error of each group allows modelers to make better-informed decisions, 
particularly regarding
minority groups. 
Generalization bounds for a 
group-specific model $\hat{h}_{i}$ and a shared model $\hat{h}$ can be used to bound risk for
group $i$,
which can be used for
\emph{model selection} (i.e., group $i$ can select between $\hat{h}$ and $\hat{h}_{i}$ with confidence).
It is known that, given infinite data, individual models are always preferable, and the degree of suboptimality of a shared model can be bounded using transfer learning techniques; however, for data-hungry models, in particular with sparse data for minority groups, a better understanding of the interplay between generalization error and the negative impacts of majority group data on minority group performance are vital. 

We also envision more sophisticated applications of our
bounds.
For example,
if some smaller groups are more similar to minority group $i$ than a majority group, 
a shared model $\hat{h}$ optimizing,
say, \emph{utilitarian malfare},
may perform poorly
for group $i$, 
but perhaps a 
better-performing $\hat{h}'$
would arise from
optimizing
a 
\emph{more egalitarian} malfare function (i.e., higher $p$ power-mean), or one that emphasizes similar groups (through the weights vector $\wv$).
Group-fair learning methods can be combined with other aspects of model selection, 
such as
feature and hyperparameter selection, where the bias-variance tradeoff plays a significant role.
Our bounds indicate that we can provably learn a more complex shared model without overfitting, and our analysis enables rigorous model selection guarantees
, both for individual group risks and for malfare objectives.
We are hopeful that future work explores these model-search questions and other applications of our methods.

\subsection*{Acknowledgments}
This research was made possible in part by the generous support of the Ford Foundation and the MacArthur Foundation.
Cyrus Cousins also wishes to acknowledge
the Center for Data Science at the University of Massachusetts Amherst, where part of this work was
conducted
under
a
postdoctoral fellowship. 
We would also like to thank Kweku Kwegyir-Aggrey for his insights and feedback on early versions of this work.

\clearpage

\bibliographystyle{apalike}
\bibliography{bibliography,bib2}

\begin{thebibliography}{}

\bibitem[Abernethy et~al., 2022]{abernethy2022active}
Abernethy, J.~D., Awasthi, P., Kleindessner, M., Morgenstern, J., Russell, C.,
  and Zhang, J. (2022).
\newblock Active sampling for min-max fairness.
\newblock In {\em International Conference on Machine Learning}, volume 162.

\bibitem[Agrawal et~al., 2018]{agrawal2018rewriting}
Agrawal, A., Verschueren, R., Diamond, S., and Boyd, S. (2018).
\newblock A rewriting system for convex optimization problems.
\newblock {\em Journal of Control and Decision}, 5(1):42--60.

\bibitem[Barman et~al., 2020]{barman2020tight}
Barman, S., Bhaskar, U., Krishna, A., and Sundaram, R.~G. (2020).
\newblock Tight approximation algorithms for {$p$}-mean welfare under
  subadditive valuations.
\newblock {\em Leibniz International Proceedings in Informatics, LIPIcs}, 173.

\bibitem[Bartlett et~al., 2005]{bartlett2005local}
Bartlett, P.~L., Bousquet, O., and Mendelson, S. (2005).
\newblock Local {R}ademacher complexities.
\newblock {\em The Annals of Statistics}, 33(4):1497--1537.

\bibitem[Bartlett and Mendelson, 2002]{bartlett2002rademacher}
Bartlett, P.~L. and Mendelson, S. (2002).
\newblock Rademacher and {G}aussian complexities: Risk bounds and structural
  results.
\newblock {\em Journal of Machine Learning Research}, 3(Nov):463--482.

\bibitem[Bossert and Kamaga, 2020]{bossert2020axiomatization}
Bossert, W. and Kamaga, K. (2020).
\newblock An axiomatization of the mixed utilitarian-maximin social welfare
  orderings.
\newblock {\em Economic Theory}, 69(2):451--473.

\bibitem[Boyd and Vandenberghe, 2004]{boyd2004convex}
Boyd, S.~P. and Vandenberghe, L. (2004).
\newblock {\em Convex optimization}.
\newblock Cambridge university press.

\bibitem[Buolamwini and Gebru, 2018]{buolamwini2018gender}
Buolamwini, J. and Gebru, T. (2018).
\newblock Gender shades: Intersectional accuracy disparities in commercial
  gender classification.
\newblock In {\em Conference on fairness, accountability and transparency},
  pages 77--91. PMLR.

\bibitem[Caruana, 1997]{caruana1997multitask}
Caruana, R. (1997).
\newblock Multitask learning.
\newblock {\em Machine learning}, 28:41--75.

\bibitem[Cavazos et~al., 2020]{cavazos2020accuracy}
Cavazos, J.~G., Phillips, P.~J., Castillo, C.~D., and O’Toole, A.~J. (2020).
\newblock Accuracy comparison across face recognition algorithms: Where are we
  on measuring race bias?
\newblock {\em IEEE Transactions on Biometrics, Behavior, and Identity
  Science}.

\bibitem[Chen et~al., 2018]{chen_why}
Chen, I.~Y., Johansson, F.~D., and Sontag, D. (2018).
\newblock Why is my classifier discriminatory?
\newblock In {\em Proceedings of the 32nd {International} {Conference} on
  {Neural} {Information} {Processing} {Systems}}, {NIPS}'18, pages 3543--3554,
  Red Hook, NY, USA. Curran Associates Inc.

\bibitem[Cortes et~al., 2020]{cortes2020agnostic}
Cortes, C., Mohri, M., Gonzalvo, J., and Storcheus, D. (2020).
\newblock Agnostic learning with multiple objectives.
\newblock {\em Advances in Neural Information Processing Systems}, 33.

\bibitem[Cousins, 2021]{cousins2021axiomatic}
Cousins, C. (2021).
\newblock An axiomatic theory of provably-fair welfare-centric machine
  learning.
\newblock In {\em Advances in Neural Information Processing Systems}.

\bibitem[Cousins, 2022]{cousins2022uncertainty}
Cousins, C. (2022).
\newblock Uncertainty and the social planner’s problem: {W}hy sample
  complexity matters.
\newblock In {\em Proceedings of the 2022 ACM Conference on Fairness,
  Accountability, and Transparency}.

\bibitem[Cousins, 2023a]{cousins2023algorithms}
Cousins, C. (2023a).
\newblock Algorithms and analysis for optimizing robust objectives in fair
  machine learning.
\newblock In {\em Columbia Workshop on Fairness in Operations and AI}. Columbia
  University.

\bibitem[Cousins, 2023b]{cousins2023revisiting}
Cousins, C. (2023b).
\newblock Revisiting fair-{PAC} learning and the axioms of cardinal welfare.
\newblock In {\em Artificial Intelligence and Statistics (AISTATS)}.

\bibitem[Cousins et~al., 2022]{cousins2022faire3}
Cousins, C., Asadi, K., and Littman, M.~L. (2022).
\newblock Fair {E$^3$}: {E}fficient welfare-centric fair reinforcement
  learning.
\newblock In {\em 5th Multidisciplinary Conference on Reinforcement Learning
  and Decision Making (RLDM)}.

\bibitem[Cousins and Riondato, 2020]{cousins2020sharp}
Cousins, C. and Riondato, M. (2020).
\newblock Sharp uniform convergence bounds through empirical centralization.
\newblock {\em Advances in Neural Information Processing Systems}, 33.

\bibitem[Cousins et~al., 2023a]{cousins2023dividing}
Cousins, C., Viswanathan, V., and Zick, Y. (2023a).
\newblock Dividing good and better items among agents with submodular
  valuations.
\newblock In {\em International Conference on Web and Internet Economics}.
  Springer.

\bibitem[Cousins et~al., 2023b]{cousins2023good}
Cousins, C., Viswanathan, V., and Zick, Y. (2023b).
\newblock The good, the bad and the submodular: {F}airly allocating mixed manna
  under order-neutral submodular preferences.
\newblock In {\em International Conference on Web and Internet Economics}.
  Springer.

\bibitem[Dalton, 1920]{dalton1920measurement}
Dalton, H. (1920).
\newblock The measurement of the inequality of incomes.
\newblock {\em The Economic Journal}, 30(119):348--361.

\bibitem[Debreu, 1959]{debreu1959topological}
Debreu, G. (1959).
\newblock Topological methods in cardinal utility theory.
\newblock {\em Cowles Foundation Discussion Papers}, 76.

\bibitem[Deschamps and Gevers, 1978]{deschamps1978leximin}
Deschamps, R. and Gevers, L. (1978).
\newblock Leximin and utilitarian rules: a joint characterization.
\newblock {\em Journal of Economic Theory}, 17(2):143--163.

\bibitem[Diamond and Boyd, 2016]{diamond2016cvxpy}
Diamond, S. and Boyd, S. (2016).
\newblock {CVXPY}: {A} {P}ython-embedded modeling language for convex
  optimization.
\newblock {\em Journal of Machine Learning Research}, 17(83):1--5.

\bibitem[Diana et~al., 2021]{diana2021minimax}
Diana, E., Gill, W., Kearns, M., Kenthapadi, K., and Roth, A. (2021).
\newblock Minimax group fairness: Algorithms and experiments.
\newblock In {\em Proceedings of the 2021 AAAI/ACM Conference on AI, Ethics,
  and Society}, pages 66--76.

\bibitem[Domahidi et~al., 2013]{Domahidi2013ecos}
Domahidi, A., Chu, E., and Boyd, S. (2013).
\newblock {ECOS}: {A}n {SOCP} solver for embedded systems.
\newblock In {\em European Control Conference (ECC)}, pages 3071--3076.

\bibitem[Dong and Cousins, 2022]{dong2022decentering}
Dong, E. and Cousins, C. (2022).
\newblock Decentering imputation: Fair learning at the margins of demographics.
\newblock In {\em Queer in AI Workshop @ ICML}.

\bibitem[Dwork et~al., 2018]{dwork2018decoupled}
Dwork, C., Immorlica, N., Kalai, A.~T., and Leiserson, M. (2018).
\newblock Decoupled classifiers for group-fair and efficient machine learning.
\newblock In {\em Conference on fairness, accountability and transparency},
  pages 119--133. PMLR.

\bibitem[Gajdos and Weymark, 2005]{gajdos2005multidimensional}
Gajdos, T. and Weymark, J.~A. (2005).
\newblock Multidimensional generalized {G}ini indices.
\newblock {\em Economic Theory}, 26(3):471--496.

\bibitem[Gorman, 1968]{gorman1968structure}
Gorman, W.~M. (1968).
\newblock The structure of utility functions.
\newblock {\em The Review of Economic Studies}, 35(4):367--390.

\bibitem[Hoeffding, 1963]{hoeffding1963probability}
Hoeffding, W. (1963).
\newblock Probability inequalities for sums of bounded random variables.
\newblock {\em Journal of the American Statistical Association},
  58(301):13--30.

\bibitem[Hu et~al., 2018]{hu2018does}
Hu, W., Niu, G., Sato, I., and Sugiyama, M. (2018).
\newblock Does distributionally robust supervised learning give robust
  classifiers?
\newblock In {\em International Conference on Machine Learning}, pages
  2029--2037. PMLR.

\bibitem[Lahoti et~al., 2020]{lahoti2020fairness}
Lahoti, P., Beutel, A., Chen, J., Lee, K., Prost, F., Thain, N., Wang, X., and
  Chi, E. (2020).
\newblock Fairness without demographics through adversarially reweighted
  learning.
\newblock {\em Advances in neural information processing systems}, 33:728--740.

\bibitem[Martinez et~al., 2020]{martinez2020minimax}
Martinez, N., Bertran, M., and Sapiro, G. (2020).
\newblock Minimax {P}areto fairness: A multi objective perspective.
\newblock In {\em International Conference on Machine Learning}, pages
  6755--6764. PMLR.

\bibitem[McDiarmid, 1989]{mcdiarmid1989method}
McDiarmid, C. (1989).
\newblock On the method of bounded differences.
\newblock {\em Surveys in combinatorics}, 141(1):148--188.

\bibitem[Mehrabi et~al., 2021]{mehrabi2021survey}
Mehrabi, N., Morstatter, F., Saxena, N., Lerman, K., and Galstyan, A. (2021).
\newblock A survey on bias and fairness in machine learning.
\newblock {\em ACM computing surveys (CSUR)}, 54(6):1--35.

\bibitem[Mitzenmacher and Upfal, 2017]{mitzenmacher2017probability}
Mitzenmacher, M. and Upfal, E. (2017).
\newblock {\em Probability and computing: Randomization and probabilistic
  techniques in algorithms and data analysis}.
\newblock Cambridge University Press, second edition.

\bibitem[Obermeyer et~al., 2019]{obermeyer2019dissecting}
Obermeyer, Z., Powers, B., Vogeli, C., and Mullainathan, S. (2019).
\newblock Dissecting racial bias in an algorithm used to manage the health of
  populations.
\newblock {\em Science}, 366(6464):447--453.

\bibitem[O'Donoghue et~al., 2016]{scs}
O'Donoghue, B., Chu, E., Parikh, N., and Boyd, S. (2016).
\newblock Conic optimization via operator splitting and homogeneous self-dual
  embedding.
\newblock {\em Journal of Optimization Theory and Applications},
  169(3):1042--1068.

\bibitem[Oren et~al., 2019]{oren2019distributionally}
Oren, Y., Sagawa, S., Hashimoto, T.~B., and Liang, P. (2019).
\newblock Distributionally robust language modeling.
\newblock In {\em Proceedings of the 2019 Conference on Empirical Methods in
  Natural Language Processing and the 9th International Joint Conference on
  Natural Language Processing (EMNLP-IJCNLP)}, pages 4227--4237.

\bibitem[Pigou, 1912]{pigou1912wealth}
Pigou, A.~C. (1912).
\newblock {\em Wealth and welfare}.
\newblock Macmillan and Company, limited.

\bibitem[Sagawa et~al., 2019]{sagawa2019distributionally}
Sagawa, S., Koh, P.~W., Hashimoto, T.~B., and Liang, P. (2019).
\newblock Distributionally robust neural networks.
\newblock In {\em International Conference on Learning Representations}.

\bibitem[Schneider and Kim, 2020]{schneider2020utilitarian}
Schneider, M. and Kim, B.-C. (2020).
\newblock The utilitarian-maximin social welfare function and anomalies in
  social choice.
\newblock {\em Southern Economic Journal}, 87(2):629--646.

\bibitem[Shalev-Shwartz and Ben-David, 2014]{shalev2014understanding}
Shalev-Shwartz, S. and Ben-David, S. (2014).
\newblock {\em Understanding machine learning: From theory to algorithms}.
\newblock Cambridge University Press.

\bibitem[Shekhar et~al., 2021]{shekhar2021adaptive}
Shekhar, S., Fields, G., Ghavamzadeh, M., and Javidi, T. (2021).
\newblock Adaptive sampling for minimax fair classification.
\newblock {\em Advances in Neural Information Processing Systems}, 34.

\bibitem[Thomas et~al., 2019]{thomas2019preventing}
Thomas, P.~S., da~Silva, B.~C., Barto, A.~G., Giguere, S., Brun, Y., and
  Brunskill, E. (2019).
\newblock Preventing undesirable behavior of intelligent machines.
\newblock {\em Science}, 366(6468):999--1004.

\bibitem[Viswanathan and Zick, 2023]{viswanathan2023general}
Viswanathan, V. and Zick, Y. (2023).
\newblock A general framework for fair allocation under matroid rank
  valuations.
\newblock In {\em Proceedings of the 24th ACM Conference on Economics and
  Computation}, pages 1129--1152.

\bibitem[Wang et~al., 2021]{benefit_of_splitting}
Wang, H., Hsu, H., Diaz, M., and Calmon, F.~P. (2021).
\newblock To split or not to split: The impact of disparate treatment in
  classification.
\newblock {\em IEEE Transactions on Information Theory}, 67(10):6733--6757.

\bibitem[Weymark, 1981]{weymark1981generalized}
Weymark, J.~A. (1981).
\newblock Generalized {G}ini inequality indices.
\newblock {\em Mathematical Social Sciences}, 1(4):409--430.

\bibitem[Zhang et~al., 2020]{zhang2020generalization}
Zhang, C., Tao, D., Hu, T., and Liu, B. (2020).
\newblock Generalization bounds of multitask learning from perspective of
  vector-valued function learning.
\newblock {\em IEEE Transactions on Neural Networks and Learning Systems},
  32(5):1906--1919.

\bibitem[Zhang and Yang, 2018]{zhang2018overview}
Zhang, Y. and Yang, Q. (2018).
\newblock An overview of multi-task learning.
\newblock {\em National Science Review}, 5(1):30--43.

\bibitem[Zhang and Yang, 2021]{zhang2021survey}
Zhang, Y. and Yang, Q. (2021).
\newblock A survey on multi-task learning.
\newblock {\em IEEE Transactions on Knowledge and Data Engineering},
  34(12):5586--5609.

\end{thebibliography}

\newpage

\onecolumn
\appendix

\ccnote{New file for appendices?}

\section{Proofs}
\label{appx:sec:proofs}

\ccnote{Restatables}

We now show \cref{thm:bounds-theoretical}.
\thmboundstheoretical*
\begin{proof}
We begin by proving part 1 and then we prove part 2 as a consequence.

\smallskip

We now show part 1. 
Recall that $\etav_{i} \doteq 2\Rade_{\bm{m}_{i}}(\loss \circ \HC, \ProbDist_{i} ) + \epsv_{i}$. With probability at least $1 - 2\delta$, for all $h \in \HC$, it holds that
\[
\abs{ \Risk(h, \ProbDist_{i}) - \ERisk(h, \bm{z}_{i}) } \leq \etav_{i} \enspace.
\]
This is a textbook application of McDiarmid's bounded difference inequality, using twice the Rademacher average to bound the expected supremum deviation, i.e., the upper and lower tails of \eqref{eq:rade-textbook}.

We could use this directly to show a weaker version of the result, however to show the stated form, we need only one tail of the above, which is used to bound generalization error of the (unknown) $\hat{h}$, and also one tail of the simple Hoeffding’s inequality tail bound \eqref{eq:hoeff-bound}. 

Now, suppose some arbitrary but fixed $h'$ that realizes the infimum of \eqref{eq:theoretical-restricted-class}, i.e., 
\[
h' \in \argmin_{\!h' \in \HC\!} \Malfare\!\left( j \mapsto \begin{cases} j \! \neq \! i \!\! & \ERisk(h', \bm{z}_{j}) \\ j \! = \! i \!\! & \Risk(h', \ProbDist_{i}) + \epsv_{i} \end{cases} \right)
\]
(technically, $h'$ may be in $\HC$ or a limit of a sequence of functions in $\HC$
).
Recalling $\epsv_{i} \doteq \frange\sqrt{\frac{\ln \frac{1}{\delta}}{2\bm{m}_{i}}}$, we obtain by Hoeffding's inequality that, 
with probability at least $1-2\delta$, it holds
\[
\abs{ \Risk(h', \ProbDist_{i}) - \ERisk(h', \bm{z}_{i}) } \leq \epsv_{i} \enspace.
\]

Therefore, when these bounds hold, we have
\begin{align*}
\hspace{-1.5cm}
\Malfare\!\left( j \mapsto \begin{cases} j \! \neq \! i \!\! & \ERisk( \hat{h}, \bm{z}_{j}) \\ j \! = \! i \!\! & \Risk(\hat{h}, \ProbDist_{i}) \! - \! \etav_{i} \end{cases} 
    \right)
 \hspace{-4cm}
&\hspace{4cm}
\leq \Malfare\!\left( j \mapsto \ERisk( \hat{h}, \bm{z}_{j}) 
    \right) & \begin{tabular}{r} W.h.p.: $\Risk(\hat{h}, \ProbDist_{i}) - \etav_{i} \leq \ERisk(\hat{h}, \bm{z}_{i})$ \\ Monotonicity of $\Malfare(\cdot)$ \end{tabular} \\
 &=  \inf_{\!h' \in \HC\!} \Malfare\!\left( j \mapsto \ERisk(h', \bm{z}_{j}) 
    \right) & \begin{tabular}{r} By Definition \end{tabular} \\
 &\leq \inf_{\!h' \in \HC\!} \Malfare\!\left( j \mapsto \begin{cases} j \! \neq \! i \!\! & \ERisk(h', \bm{z}_{j}) \\ j \! = \! i \!\! & \Risk(h', \ProbDist_{i}) \! + \! \epsv_{i} \end{cases} 
    \right) \!\enspace. & \hspace{-0.5cm} \begin{tabular}{r} W.h.p.: $\ERisk(h', \bm{z}_{i}) \leq \Risk(h', \ProbDist_{i}) + \epsv_{i}$ \\ Monotonicity of $\Malfare(\cdot)$ \end{tabular}
\end{align*}
We may thus conclude with probability at least $1 - 4\delta$ that $\hat{h} \in \HC_{i}^{*}$ (by definition).
However, observe that both the McDiarmid (Rademacher) and Hoeffding bounds required only one tail each, and thus a more careful analysis yields the guarantee with probability at least $1 - 2\delta$.


\smallskip

We now show part 2.
By part 1, we have that $\hat{h} \in \HC_{i}^{*}$ with probability at least $1 - 2\delta$.
Then we apply the standard 2-tailed Rademacher bound with McDiarmid's inequality over the restricted class $\HC_{i}^{*}$, i.e., we have
\[
\Prob_{\bm{z}_{i} \distributed \ProbDist_{i}^{\bm{m}_{i}}} \left( \sup_{h \in \HC_{i}^{*}} 
\abs{ \Risk(h, \ProbDist_{i}) - \ERisk(h, \bm{z}_{i}) } \leq \etav_{i} \right) \leq 1 - 2\delta
\enspace.
\]
The union bound then yields the desideratum.
\end{proof}


We now show \cref{thm:bounds-empirical}.
\thmboundsempirical*
\begin{proof}
We begin by proving part 1, and we then show part 2 as a consequence.

\smallskip

We now show part 1.
First, we apply part 1 of \cref{thm:bounds-theoretical} (2 tails).
We will then 
argue that
\[
\Prob_{\bm{z}_{i} \distributed \ProbDist_{i}^{\bm{m}_{i}}}
    \left( \HC_{i}^{*} \subseteq \hat{\HC}_{i} \right) \geq 1 - 2\delta
\enspace,
\]
which holds for similar reasons (a 1-tail Rademacher bound for $\HC$, and a 1-tail Hoeffding bound for $\hat{h}$, both the opposite tails bounded in  part 1 of \cref{thm:bounds-theoretical}). 
The result then follows via union bound.

In particular,
recall \eqref{eq:theoretical-restricted-class}
\begin{equation*}
\HC_{i}^{*} \doteq \left\{ h \in \HC \,\middle|\, \Malfare\!\left( j \mapsto \begin{cases} j \! \neq \! i \!\! & \ERisk(h, \bm{z}_{j}) \\ j \! = \! i \!\! & \Risk(h, \ProbDist_{i}) - \etav_{i}^{} \end{cases} 
	 \right) \leq \inf_{\!h' \in \HC\!} \Malfare\!\left( j \mapsto \begin{cases} j \! \neq \! i \!\! & \ERisk(h', \bm{z}_{j}) \\ j \! = \! i \!\! & \Risk(h', \ProbDist_{i}) + \epsv_{i} \end{cases} 
	 \right) \! \right\} \enspace,
\end{equation*}
and also \eqref{eq:empirical-restricted-class}
\begin{equation*}
\hat{\HC}_{i} \doteq \left\{ h \in \hat{\HC} \,\middle|\, \Malfare\!\left( j \mapsto \begin{cases} j \! \neq \! i \!\! & \ERisk(h, \bm{z}_{j}) \\ j \! = \! i \!\! & \ERisk(h, \bm{z}_{i}) - 2\hat{\etav}_{i} \end{cases} 
    \right) \leq \inf_{\!h' \in \HC\!} \Malfare\!\left( j \mapsto \begin{cases} j \! \neq \! i \!\! & \ERisk(h', \bm{z}_{j}) \\ j \! = \! i \!\! & \ERisk(h', \bm{z}_{i}) + 2\epsv_{i} \end{cases} 
    \right) \! \right\}
    \enspace.
\end{equation*}

Now, observe that by McDiarmid's inequality, by essentially the same argument as in \eqref{eq:rade-textbook}, it holds that
{\small
\[
\! \Prob_{\bm{z}_{i} \hspace{-0.1em} \distributed \hspace{-0.1em} \ProbDist_{i}^{ \hspace{-0.1em} \bm{m}_{i\!}}
    } \!\! \left(
     \sup_{\mathclap{h \in \HC}} \ERisk(h, \bm{z}_{i}) \! - \! \Risk(h, \ProbDist_{i}) 
     + \etav_{i} 
     > 2\hat{\etav}_{i}
  \!\hspace{-0.1em}
  \right)
  = \!\!\!\!\! \Prob_{\bm{z}_{i} \hspace{-0.1em} \distributed \hspace{-0.1em} \ProbDist_{i}^{ \hspace{-0.1em} \bm{m}_{i\!}}
    }\!\! \left(
     \sup_{\mathclap{h \in \HC}} \ERisk(h, \bm{z}_{i}) \! - \! \Risk(h, \ProbDist_{i})
     + 2\Rade_{\bm{m}_{i}\!}(\loss \circ \HC, \ProbDist_{i}) 
     >
     4\ERade_{\bm{m}_{i}\!}(\loss \circ \HC, \bm{z}_{i}) \! + \! 3\epsv_{i}
  \!\hspace{-0.1em}
  \right)
  < \delta
  \hspace{-0.1em}\enspace. 
\]%
}%
%
We thus have that,
with probability at least $1 - \delta$,
for all $h \in \HC_{i}^{*}$, 
\[
\Malfare\!\left( j \mapsto \begin{cases} j \! \neq \! i \!\! & \ERisk(h, \bm{z}_{j}) \\ j \! = \! i \!\! & \ERisk(h, \bm{z}_{i}) - 2\hat{\etav}_{i} \end{cases} 
    \right) \leq \Malfare\!\left( j \mapsto \begin{cases} j \! \neq \! i \!\! & \ERisk(h, \bm{z}_{j}) \\ j \! = \! i \!\! & \Risk(h, \ProbDist_{i}) - \etav_{i} \end{cases} 
    \right) \enspace,
\]
and similarly,
with probability at least $1 - \delta$
by the Hoeffding bound \eqref{eq:hoeff-bound} on $h'$,
we have
\[
\inf_{\!h' \in \HC\!} \Malfare\!\left( j \mapsto \begin{cases} j \! \neq \! i \!\! & \ERisk(h', \bm{z}_{j}) \\ j \! = \! i \!\! & \Risk(h', \ProbDist_{i}) + \epsv_{i} \end{cases} 
    \right)
\leq 
\inf_{\!h' \in \HC\!} \Malfare\!\left( j \mapsto \begin{cases} j \! \neq \! i \!\! & \ERisk(h', \bm{z}_{j}) \\ j \! = \! i \!\! & \ERisk(h', \bm{z}_{i}) + 2\epsv_{i} \end{cases} 
    \right) \enspace,
\]
where both steps apply monotonicity of $\Malfare(\cdot)$.

From this, we may conclude that, with probability at least $1 - 2\delta$, for each $h \in \HC$, if  the constraint in \eqref{eq:theoretical-restricted-class} is satisfied, then the constraint \eqref{eq:empirical-restricted-class} is satisfied, thus ${\HC}_{i}^{*} \subseteq \hat{\HC}_{i}$.
The union bound over all tail bounds above then yields part 1.

\smallskip

We now show part 2.
This result essentially follows the structure of part 2 of \cref{thm:bounds-theoretical}.
However, we now start with part 1 above, which allows us to conclude that $\hat{h} \in \hat{\HC}_{i}$ with probability at least $1 - 4\delta$, and then apply the standard empirical Rademacher bounds, i.e., 2 tails of \eqref{eq:rade-textbook} (we require only the upper and lower bounds to the supremum deviation, not the bound on the Rademacher average itself), to $ \hat{\HC}_{i}$ (rather than to $ {\HC}_{i}^{*}$). 
Taking the union bound over all events then yields the desideratum.
%
%
\end{proof}

We now show \cref{coro:emp-malfare-bounds}.
\coroempmalfarebounds*
\begin{proof}

For both
results, we apply part 2 of \cref{thm:bounds-empirical} to each group $i$, which by union bound gives a result with probability at least $1 - 6\NGroups\delta$.
However, careful accounting reveals that we only require one tail of the final Rademacher bound of \cref{thm:bounds-empirical}~part~2, i.e., we require $\Risk(\hat{h}, \ProbDist_{j}) \leq \ERisk(\hat{h}, \bm{z}_{j}) + 2\ERade_{\bm{m}_{j}}(\hat{\HC}_{j}, \bm{z}_{j}) + 2\epsv_{j}$, \emph{but not} $\ERisk(\hat{h}, \bm{z}_{j}) \leq \Risk(\hat{h}, \ProbDist_{j}) + 2\ERade_{\bm{m}_{j}}(\hat{\HC}_{j}, \bm{z}_{j}) + 2\epsv_{j}$, thus 
we begin with tail bounds that hold with probability at least $1 - 5\NGroups\delta$.

We now show part 1. 
Subject to all tail bounds holding, we have
for all $j \in 1,\dots,\NGroups$ that $\Risk(\hat{h}, \ProbDist_{j}) \leq \ERisk(\hat{h}, \bm{z}_{j}) + 2\ERade_{\bm{m}_{j}}(\hat{\HC}_{j}, \bm{z}_{j}) + 2\epsv_{j}$, thus by monotonicity of $\Malfare(\cdot)$, we have
\[
\Malfare\left(j \mapsto \Risk(\hat{h}, \ProbDist_{j}) 
    \right)
\leq \Malfare\left(j \mapsto \ERisk(\hat{h}, \bm{z}_{j}) + 2\ERade_{\bm{m}_{j}}(\hat{\HC}_{j}, \bm{z}_{j}) + 2\epsv_{j} 
    \right) 
\enspace.
\]
Applying the Lipschitz property then yields the final portion of part 1. 

\smallskip

We now show part 2. 
First, observe that 
$
\Malfare\left(j \mapsto \Risk(h^{*}, \ProbDist_{j}) 
    \right)
\leq \Malfare\left(j \mapsto \Risk(\hat{h}, \ProbDist_{j}) 
    \right)
$ 
by definition.
For the remaining inequality, we introduce one new tail bound for each group $j$, in particular, a 1-tail Hoeffding bound of 
\[
\Prob_{\bm{z}_{j} \distributed \ProbDist_{j}^{\bm{m}_{j}}}\left(\Risk(h^{*}, \ProbDist_{j}) \leq \ERisk(h^{*}, \bm{z}_{j}) + \epsv_{j} \right) \geq 1 - \delta
\enspace.
\]
This seems familiar, but it is not quite the same as the 2-tail Hoeffding bound on each $\Risk(h', \ProbDist_{j})$ used by
\cref{thm:bounds-theoretical,thm:bounds-empirical}, thus this tail bound must be counted separately.
Now, we substitute into the result of part 1, again applying monotonicity, to get
\[
\Malfare\left(j \mapsto \Risk(\hat{h}, \ProbDist_{j}) 
    \right)
\leq \Malfare\left(j \mapsto \Risk(h^{*}, \ProbDist_{j}) + 2\ERade_{\bm{m}_{j}}(\hat{\HC}_{j}, \bm{z}_{j}) + 3\epsv_{j} 
    \right)
\enspace.
\]
Applying the Lipschitz property then yields the final portion of part 2. 
By union bound, we may conclude the result with probability at least $1 - 6\NGroups\delta$.
\end{proof}

We now show \cref{lemma:cvx}.
\lemmacvx*
\begin{proof}
We first show that the restricted parameter spaces of ${\HC}^{*}_{i}$ and $\hat{\HC}_{i}$ are convex sets.

The crux of
this result
is to show
that $\Malfare\bigl(j \mapsto f_{j}(\bm{\beta})\bigr)$ is quasiconvex, where $f_{j}(x)$ represents $\ERisk(h_{\bm{\beta}}, \bm{z}_{j}) - \bm{c}_{j}$ or $\Risk(h_{\bm{\beta}}, \ProbDist_{j}) - \bm{c}_{j}$ for some constant $\bm{c} \in \R^{\NGroups}$.
This indeed holds, so long as $f_{j}(\bm{\beta})$ is quasiconvex.
First note that convexity of 
\emph{loss} immediately implies convexity of (empirical) \emph{risk}.
Now, by standard compositional rules, since we assume $\Malfare(\cdot)$ to be quasiconvex and monotonic, we conclude that $\Malfare(j \mapsto f_{j}(\bm{\beta}))$ is quasiconvex in $\bm{\beta} \in \bm{B}$.

Now, converting $\HC$ to $\bm{B}$, observe that the parameter spaces associated with both $\HC_{i}^{*}$ in \eqref{eq:theoretical-restricted-class}
\begin{equation*}
\left\{
    \bm{\beta} \in \bm{B}
\,\middle|\, \Malfare\!\left( j \mapsto \begin{cases} j \! \neq \! i \!\! & \ERisk(h_{\bm{\beta}}, \bm{z}_{j}) \\ j \! = \! i \!\! & \Risk(h_{\bm{\beta}}, \ProbDist_{i}) - \etav_{i}^{} \end{cases} 
	 \right) \leq \inf_{\!h' \in \HC\!} \Malfare\!\left( j \mapsto \begin{cases} j \! \neq \! i \!\! & \ERisk(h', \bm{z}_{j}) \\ j \! = \! i \!\! & \Risk(h', \ProbDist_{i}) + \epsv_{i} \end{cases} 
	 \right) \! \right\}
  \enspace,
\end{equation*}
and also $\hat{\HC}_{i}$ in \eqref{eq:empirical-restricted-class}
\begin{equation*}
\left\{
    \bm{\beta} \in \bm{B} 
\,\middle|\, \Malfare\!\left( j \mapsto \begin{cases} j \! \neq \! i \!\! & \ERisk(h_{\bm{\beta}}, \bm{z}_{j}) \\ j \! = \! i \!\! & \ERisk(h_{\bm{\beta}}, \bm{z}_{i}) - 2\hat{\etav}_{i} \end{cases} 
    \right) \leq \inf_{\!h' \in \HC\!} \Malfare\!\left( j \mapsto \begin{cases} j \! \neq \! i \!\! & \ERisk(h', \bm{z}_{j}) \\ j \! = \! i \!\! & \ERisk(h', \bm{z}_{i}) + 2\epsv_{i} \end{cases} 
    \right) \! \right\}
    \enspace,
\end{equation*}
are subsets of the convex set $\bm{B}$.
In particular, the RHS of the condition is \emph{constant} in $\bm{\beta}$, and as above, the LHS is quasiconvex in $\bm{\beta}$, thus both restricted parameter spaces are convex sets.


Now,
note that
once we determine the parameter space to be convex, Monte-Carlo Rademacher averages can be efficiently computed via standard convex optimization techniques, e.g., first-order methods to maximize a linear objective on a convex set.
Key to this observation is that we assumed $g \circ \HC$ to be an affine function family, thus even after multiplying terms by $\pm 1$ in the Monte-Carlo Rademacher average \eqref{eq:lin-mcera}, the objective of the supremum remains convex.

Finally, observe that if $\Malfare(\cdot)$ is convex and monotonically increasing, then the EMM objective is also convex. This follows from standard compositional rules, see discussion following \citet{boyd2004convex} equation~(3.15).  EMM then reduces to minimizing a convex function on a convex set.
\end{proof}

\section{Implementation details}

All code used to generate the results in this paper are available upon request.
The computation of 
each supremum in \eqref{eq:lin-mcera}, i.e., $\ERade{}^n_{m}(g \circ \HC, \bm{z}_{i}; \vsigma)$ 
and $\ERade{}^n_{m}(g \circ \hat{\HC}_{i}, \bm{z}_{i}; \vsigma)$
optimize linear functions of $\beta$.
However, since the restricted hypothesis constraints are convex functions (see~\cref{lemma:cvx}) over the 
parameter space, we need to use solvers that can handle 
nonlinear convex constraints.
For this reason, we use either the ECOS \citep{Domahidi2013ecos} or SCS \citep{scs} algorithms available in CVXPY \citep{diamond2016cvxpy, agrawal2018rewriting}. These algorithms are also able to compute the upper bound of the restricted hypothesis class constraint itself, which minimizes an objective  which is 
dependent on the loss $\ell$. 

\end{document}